\documentclass[review]{elsarticle}

\usepackage{hyperref}

\journal{Journal of Artificial Intelligence}

\usepackage{bm}
\usepackage{graphicx}
\usepackage{subfigure}
\usepackage{amsmath}
\usepackage{amsthm}
\usepackage{amssymb}
\usepackage{algorithm,algorithmic}
\usepackage{thm-restate}

\usepackage[usenames,dvipsnames]{xcolor}
\usepackage{ifthen}

\newcommand{\compilehidecomments}{false}
\ifthenelse{ \equal{\compilehidecomments}{true} }{%
	\newcommand{\wei}[1]{}
	\newcommand{\siwei}[1]{}
}{
	\newcommand{\wei}[1]{{\color{blue!50!black}  [\text{Wei:} #1]}}
	\newcommand{\siwei}[1]{{\color{brown!60!black} [\text{Siwei:} #1]}}
}

\newtheorem{fact}{Fact}
\newtheorem{remark}{Remark}

\newtheorem{definition}{Definition}
\newtheorem{lemma}{Lemma}

\newtheorem{assumption}{Assumption}

\newcommand{\E}{\mathbb{E}}

\newcommand{\argmax}{\operatornamewithlimits{argmax}}
\newcommand{\argmin}{\operatornamewithlimits{argmin}}
\newcommand{\rad}{\operatornamewithlimits{rad}}

\newcommand{\cM}{\mathcal{M}}

\newcommand{\I}{\mathbb{I}}

\allowdisplaybreaks

\begin{document}

\begin{frontmatter}

\title{Thompson Sampling for Combinatorial Semi-Bandits}


\author[mymainaddress]{Siwei Wang\corref{mycorrespondingauthor}}
\cortext[mycorrespondingauthor]{Corresponding author}
\ead{wangsw2020@mail.tsinghua.edu.cn}

\author[mysecondaryaddress]{Wei Chen}
\ead{weic@microsoft.com}

\address[mymainaddress]{Tsinghua University, Beijing}
\address[mysecondaryaddress]{Microsoft Research Asia, Beijing}

\begin{abstract}
In this paper, we study the application of the Thompson sampling (TS) methodology to the stochastic combinatorial multi-armed bandit (CMAB) framework. 
We first analyze the standard TS algorithm for the general CMAB model when the outcome distributions of all the base arms are independent, and obtain a distribution-dependent regret bound of $O(m\log K_{\max}\log T / \Delta_{\min})$, where $m$ is the number of base arms, $K_{\max}$ is the size of the largest super arm, $T$ is the time horizon, and $\Delta_{\min}$ is the minimum gap between the expected reward of the optimal solution and any non-optimal solution. 
%
%
This regret upper bound is better than the $O(m(\log K_{\max})^2\log T / \Delta_{\min})$ bound in prior works. 
Moreover, our novel analysis techniques can help to tighten the regret bounds of other existing UCB-based policies (e.g., ESCB), as we improve the method of counting the cumulative regret.
Then we consider the matroid bandit setting (a special class of CMAB model), where we could remove the independence assumption across arms and achieve
	a regret upper bound that matches the lower bound.
Except for the regret upper bounds, we also point out that one cannot directly replace the exact offline oracle (which takes the parameters of an offline problem instance as input and outputs the exact best action under this instance) with an approximation oracle in TS algorithm for even the classical MAB problem. 
Finally, we use some experiments to show the comparison between regrets of TS and other existing algorithms, the experimental results show that TS outperforms existing baselines. 
\end{abstract}


\end{frontmatter}


\section{Introduction}

Multi-armed bandit (MAB) \citep{Berry1985Bandit,Sutton2018Reinforcement} is a classical online learning model
	typically described as a game between a learning agent (player) and the environment with $m$ arms.
In each step, the environment generates an outcome, and the player uses a policy (or an algorithm), which
takes the feedback from the previous steps as input, to select an arm to pull.
After pulling an arm, the player receives a reward based on
	the pulled arm and the environment outcome.
In this paper, we consider stochastic MAB problem, which means the environment outcome is drawn
	from an unknown distribution \citep{Lai1985Asymptotically},
	not generated by an adversary \citep{Auer2002The}.
The goal of the player is to cumulate as much reward as possible over a total of $T$ steps ($T$ may be unknown).
The performance metric is the {\em (expected) regret}, which is the cumulative difference over $T$ steps
between always playing the arm with the optimal expected reward	and playing the arms according to the policy.

MAB models the key tradeoff between exploration --- continuing exploring new arms not observed often,
	and exploitation --- sticking to the best performing arm based on the observations so far.
A famous MAB algorithm is the upper confidence bound (UCB) policy \citep{gittins1989multi,Auer2002Finite},
	which achieves $O(m\log{T}/\Delta)$ distribution-dependent
	regret, where $\Delta$ is the minimum gap in the expected reward between an optimal arm and any non-optimal arm, and it matches
	the regret lower bound in \cite{Lai1985Asymptotically}. 

Combinatorial multi-armed bandit (CMAB) problem has recently become an active research area
	\citep{Yi2012,chen2016combinatorial,GMM14,KWAEE14,KWAS15,kveton2015combinatorial,wen2015efficient,Combes2015,ChenGeneral16,WangChen2017}.
In CMAB, the environment contains $m$ {\em base arms} (or simply arms), but
	the player needs to pull a set of base arms $S$ in each time slot, where $S$ is called a {\em super arm} (or an {\em action}).
The kind of reward and feedback varies in different settings.
In this paper, we consider the semi-bandit setting, where the feedback includes the outcomes of all base arms in the played super arm,
	and the reward is a function of $S$ and the observed outcomes of arms in $S$.
CMAB has found applications in many areas such as
	wireless networking, social networks, online advertising, etc \cite{WangChen2017,ontanon2013combinatorial,gai2010learning,qin2014contextual,ul2015flag,Chen2013Combinatorial}.
	For example, in an online advertising website, the system is conducting advertisements via the Internet, and it will receive one observation (e.g., whether the advertisement is clicked) after sending an advertisement to one user. 
However, to make the system more efficient, in each step we often choose a set of users (with size at most $K$) instead of a single user.
This forms a CMAB instance: the base arms are single users, and the super arms (actions) are all the sets of single users with size at most $K$.
Because of the widely adoption of CMAB model in real applications, it is important to investigate
	different approaches to solve CMAB problems.

An alternative approach different from UCB is the Thompson sampling (TS) approach,
	which is introduced much earlier by \cite{thompson1933likelihood}, but the theoretical analysis of the TS policy comes much later ---
	\cite{Kaufmann2012Thompson} and \cite{agrawal2012analysis} give the first regret bound for the TS policy, which essentially matches the UCB policy theoretically.
Moreover, TS policy often performs better than UCB in empirical simulation results,
	making TS an attractive policy for further studies.

TS policy follows the Bayesian inference framework to solve the MAB problems.
The unknown distribution of environment outcomes is parameterized, with an assumed prior distribution on the parameters.
TS updates the prior distribution in each step with two phases: first it uses the prior distribution to sample one set of parameters, which is used to determine the action to
	play in the current step; second it uses the feedback obtained in the current step to update the prior distribution to posterior distribution according to the
	Bayes' rule. To avoid confusion on these two kinds of random variables, in the rest of this paper, we use the word ``sample'' to denote the variable in the first phase, i.e. the random variable coming from the prior distribution on the parameters. The word ``observation'' represents the feedback random variable, which follows the unknown environment distribution.

In this paper, we study the application of the Thompson sampling approach to CMAB. The reason that we are interested in this approach is that it has good performance in experiments, especially when the reward function is linear, e.g., we found out that TS-based policy performs better than many kinds of UCB-based policy in experiments (see details in Section \ref{Section_e}). 
Another interesting thing is that TS-based policies do not require the reward function to be monotone, while UCB-based policies do need this assumption. This is because that a UCB-based policy needs to search for a parameter set (within the confidence region) such that under this set, the expected reward of all the super arms are overestimated. Only if the reward function is monotone, we can use the upper confidence bounds of all the parameters as this specific parameter set, otherwise it can be really complicate to search for that parameter set within the confidence region. Different with UCB-based policies, TS-based policies choose to draw random samples to form the parameter set and hence they do not require the monotonicity of the reward function. 
All of these make TS-based policy more competitive in real applications.

We first consider a general CMAB case similar to \cite{chen2016combinatorial}, i.e. we assume that
	(a) the problem instance satisfies a Lipschitz continuity assumption to handle non-linear
		reward functions, and
	(b) the player has access to an exact oracle for the offline optimization problem.
We use the standard TS policy together with the offline oracle,
	and refer it as the combinatorial Thompson sampling (CTS) algorithm.
CTS policy would first derive a set of parameters $\bm{\theta} = (\theta_1,\cdots,\theta_m)$ as sample set
	for the base arms, and then select the optimal super arm under parameter set $\bm{\theta}$.
	
	Comparing to the UCB-based solution in~\cite{chen2016combinatorial}, the advantages of CTS
	is that: a) we do not need
	to assume that the expected reward is monotone to the mean outcomes of the base arms; b) it has better behavior in experiments. CTS also suffers from some disadvantages. For example, in CTS policy we cannot directly replace the exact offline oracle with an approximation oracle as in \cite{chen2016combinatorial} (the regret becomes the approximation regret as well), which is useful in CMAB problems to accommodate 
		NP-hard combinatorial problems.
	However, we claim that it is because of the difference between TS-based algorithm and UCB-based algorithm. To show this, we provide a counter example for the origin MAB problem, which suffers an approximation regret of $\Theta(T)$ when using TS policy with an approximation oracle.


Another disadvantage is that we need to assume that the outcomes of all base arms are mutually independent. This is because TS policy maintains a prior distribution for every base arm's mean value $\mu_i$. Only when the arm outcome distributions are independent, we can use a simple 
	method to update those prior distributions on the parameters; otherwise the update method will be much more complicated for both the implementation and the analysis. 
This assumption is still reasonable and satisfied by many real-world applications.

The original analysis for TS on MAB model then faces a challenge in addressing the dependency issue: it essentially requires that
	different super arms be related with independent samples so that when comparing them and selecting the optimal super arm,
	the actual optimal one is selected with high probability.
But when super arms are based on the {\em same} sample set $\bm{\theta}$, dependency and correlation
	among super arms may likely fail the above high probability analysis.
%
%
%
%
One way to get around this issue is to independently derive a sample set $\bm{\theta}(S)$ for every super arm $S$ and compute its expected reward under $\bm{\theta}(S)$, and
	then select the optimal super arm.
Obviously this solution incurs exponential sampling cost and is what we want to avoid when solving CMAB.

To address the dependency challenge, we adapt an analysis of \cite{komiyama2015optimal} for
	selecting the top-$k$ arms to the general CMAB setting.
The adaptation is nontrivial since we are dealing with arbitrary combinatorial constraints
	while they only deal with super arms containing $k$ arms.
We also need an assumption that the outcomes of the base arms are independent. We will further discuss and justify this assumption shortly.



When the outcome distributions of all the base arms are independent, we show that CTS achieves $O(\sum_{i\in [m]} \log K_{\max}\log T / \Delta_{i,\min} + (2 / \varepsilon)^{2k^*})$
	distribution-dependent regret bound for some small $\varepsilon>0$, where $K_{\max}$ is the largest size of a super arm, $\Delta_{i,\min}$ is the
	minimum gap between the optimal expected reward and any non-optimal expected reward containing base arm $i$, and $k^*$ is the size of the optimal solution.
	Compare to the $O(\sum_{i\in [m]} K_{\max} \log T / \Delta_{i,\min})$ regret upper bound of applying CUCB in the general CMAB model \cite{chen2016combinatorial}, our regret bound is much better. This is because the independence assumption can significantly reduce the regret upper bound.
	For example, a similar regret bound $O(\sum_{i\in [m]} (\log K_{\max})^2 \log T / \Delta_{i,\min}) $ is achieved by a UCB-based policy
	ESCB \citep{Combes2015} in \cite{degenne2016combinatorial}, under the assumptions that the reward function is linear and outcome distributions are independent. \cite{perrault2020statistical} shows that CTS achieves the same $O(\sum_{i\in [m]} (\log K_{\max})^2 \log T / \Delta_{i,\min})$ regret upper bound when outcome distributions are independent (and it does not require the linear reward function assumption).
We adapt some basic ideas of the above results, and propose novel analysis techniques to further reduce one $\log K_{\max}$ factor in the existing regret upper bounds of CTS. 
Moreover, our novel analysis can also help to reduce one $\log K_{\max}$ factor in the regret upper bound of the ESCB policy in \cite{degenne2016combinatorial}, i.e., CTS and ESCB still have comparable regret upper bounds under the independent setting. 
However, ESCB achieves this bound with an exponential time complexity, because it needs to compute upper confidence bounds for {\em all 
	super arms} before making a selection in every time steps.
In contrast, as long as the combinatorial problem has an exact and efficient offline oracle implementation, our CTS algorithm can be implemented efficiently
	with only one call to the oracle in each time step.
Therefore, CTS algorithm is the only efficient algorithm that can achieve the above regret bound for CMAB with independent arms.
As for the exponential constant term $ O((2 / \varepsilon)^{2k^*})$, we show through an example that it is unavoidable for Thompson
	 sampling (e.g., the regret upper bound in \cite{perrault2020statistical} also contains a similar constant term). 

The independent arm assumption can be dropped if we further constrain the feasible solutions of the combinatorial problem to have a matroid 
	structure.
Matroid bandit is a special class of CMAB~\citep{KWAEE14},
	in which the base arms are the elements in the ground set and the super arms are the independent sets of a matroid. The reward of
	a super arm is the sum of outcomes of all base arms in the super arm.
We show that the regret of CTS is upper bounded by $O(\sum_{i\not \in S^*}\log {T}/\Delta_i+m / \varepsilon^4)$ for some small $\varepsilon>0$, where $S^*$ is an optimal solution
	and $\Delta_i$ is the minimum positive gap between the mean outcome of any arms in $S^*$ and the mean outcome of arm $i$. This result does not need to assume that all arm distributions are independent, and do not have a constant term exponential with $k^*$. 
It matches both the theoretical performance of the UCB-based algorithm and the lower bound given by Kveton et. al. \cite{KWAEE14}, 
	the constant term $O(m / \varepsilon^4)$ is similar to the results in Agrawal et. al. \cite{agrawal2012analysis} and appears in almost every TS analysis paper.

%
%

We further conduct empirical simulations, and show that CTS
	performs much better than the CUCB algorithm
	of~\cite{chen2016combinatorial} and C-KL-UCB algorithm based on KL-UCB of~\cite{Garivier2011The} on both matroid and non-matroid CMAB problem instances. We also compare CTS with ESCB policy given in \cite{Combes2015}, and the results show that CTS still behaves better. 


In summary, our contributions include that:
	(a) we provide a novel analysis of Thompson sampling for CMAB problems, and reduce one $\log K_{\max}$ factor in the regret upper bounds of CMAB with independence assumption in prior works \cite{degenne2016combinatorial,perrault2020statistical}; 
		(b) we show that the independence assumption can be dropped under matroid bandit problems, and obtain a regret upper bound that matches with the regret lower bound in this case; 
	(c) we demonstrate through experiments that Thompson sampling has better performance than UCB-based CMAB policies; and
	(d) we show the difficulty of incorporating approximation oracle and the inevitability of 
		the exponential constant term in the Thompson sampling approach for CMAB.

This paper is an extension of our conference paper~\citep{WC18}.
Comparing to our conference version, we improve the regret upper bound for using CTS in CMAB with independence assumption. 
As already discussed above, in this case we achieve a better regret upper bound than in \citep{Combes2015,degenne2016combinatorial,perrault2020statistical}, and our solution is an efficient algorithm while
	the solution in \citep{Combes2015,degenne2016combinatorial} requires an exponential-time computation.
We also correct a mistake in the proof of Theorem \ref{Proposition1}, i.e., Thompson Sampling policy suffers from a $\Theta(T)$ approximation regret when applying an approximation oracle to search for the pulled super arm. In our conference paper~\citep{WC18}, we ignored the case that pulling an arm can cause negative approximation regret. This makes the analysis unreliable. In this paper, we provided the corrected proof in Section \ref{Section_Proof_2}, which adapted the same idea as our conference paper, but with a more meticulous analysis.

\subsection{Related Work}

We have already mentioned a number of related works on the general context of multi-armed bandit and Thompson sampling.
In this section, we focus on the most relevant
	studies related to CMAB.

Our study follows the general CMAB framework of \cite{chen2016combinatorial}, which provides a UCB-style algorithm CUCB and show
	a $O(\sum_{i\in [m]} K_{\max}\log T / \Delta_{i,\min})$ regret bound.
CUCB policy and CTS policy both use an offline oracle and assume a Lipschitz continuity condition.
Their differences include: (a) CTS does not need the monotonicity property but CUCB requires that to compute the upper confidence bounds; 
	(b) CUCB allows an approximation oracle (and uses approximation regret), but CTS requires an exact oracle; 
	(c) CTS requires the independence assumption, and also achieves a better regret upper bound (note that even under the independence assumption, the CUCB policy still has $O(\sum_{i\in [m]} K_{\max}\log T / \Delta_{i,\min})$ regret upper bound).
Combes et. al. \cite{Combes2015} propose ESCB algorithm to solve CMAB with linear reward functions and independent arms, and prove that ESCB has the regret bound
	$O(\sum_{i\in [m]} \sqrt{K_{\max}}\log T / \Delta_{i,\min}) $, which is better than the CUCB policy.
	After that, Degenne et. al. \cite{degenne2016combinatorial} obtain an $O(\sum_{i\in [m]}( \log{K_{\max}})^2\log T / \Delta_{i,\min}) $ regret upper bound for ESCB algorithm, which is smaller than \citep{Combes2015}.
Under the same setting, we provide a better regret bound $O(\sum_{i\in [m]}\log{K_{\max}}\log T / \Delta_{i,\min}) $ for CTS, 
and our analysis can help ESCB algorithm to achieve the same regret upper bound.
CTS is better than ESCB in that ESCB requires exponential time in computation, but CTS only requires polynomial-time computation
	as long as the offline oracle is efficient.
Thus, CTS is the only known CMAB algorithm that achieves $O(\sum_{i\in [m]} \log K_{\max}\log T / \Delta_{i,\min}) $ regret bound with
	an efficient computation.

Matroid bandit is defined and studied by Kveton et. al. \cite{KWAEE14},
	who provide a UCB-based algorithm with regret bound almost exactly matches CTS algorithm.
They also prove a matching lower bound using a partition matroid bandit.

Thompson sampling has also been applied to settings with combinatorial actions.
Gopalan et. al. \cite{GMM14} study a general action space with a general feedback model,
	and provide analytical regret bounds for the exact TS policy.
However, their general model cannot be applied to our case.
In particular, they assume that the arm outcome distribution is from a known parametric family
	and the prior distribution on the parameter space is finitely supported.
We instead work on arbitrary nonparametric and unknown arm
	distributions with bounded support, and even if we work on a parametric family, we allow
		the support of prior distributions to be infinite or continuous.
The reason is that in our CMAB setting, we only need to learn the means of base arms (same as in~\cite{chen2016combinatorial}).
Moreover, their regret bounds are high probability bounds, not expected regret bounds, and
	their bounds contain a potentially very large constant, which will turn into
	a non-constant term when we convert them to expected regret bounds.

\cite{komiyama2015optimal} consider the TS-based policy for the top-$k$ CMAB problem,
	a special case of matroid bandits where the super arms are subsets of size at most $k$.
Thus we generalize top-$k$ bandits to matroid bandits, and our regret bound for matroid bandit still matches the one in \cite{komiyama2015optimal}.

\cite{perrault2020statistical} follows our conference paper \cite{WC18} and shows that CTS achieves a regret upper bound of $O(\sum_{i\in [m]}( \log{K_{\max}})^2\log T / \Delta_{i,\min})$ when the outcomes of all the base arms are independent (i.e., the same as the regret upper bound of ESCB in \cite{degenne2016combinatorial}). It is also worth to note that they do not require the reward function to be linear, while the analysis for ESCB in \cite{Combes2015,degenne2016combinatorial} needs such assumption. Our novel analysis in this manuscript adapts some of ideas in \cite{perrault2020statistical}, which further reduces one $\log K_{\max}$ term in their regret upper bounds, and does not require the linear reward function assumption as well. 

Wen et. al. \cite{wen2015efficient} analyze the regret of using TS policy for contextual CMAB problems.
The key difference between their work and ours is that they use the
	Bayesian regret metric.
Bayesian regret takes another expectation on the prior distribution of parameters,
	while our regret bound works for {\em any} given parameter.
This leads to very different analytical method, and they cannot provide distribution-dependent regret bounds.
Russo et. al. \cite{russo2016information} also use Bayesian regret to analyze the regret bounds
	of TS policy for MAB problems.
Again, due to the use of Bayesian regret, their analytical method is very different and cannot be used
	for our purpose.

%
The algorithms and analysis in these settings are also very different with ours, since the obtained information in these settings are not the same.


\section{Model and Definitions}

In this section, we introduce our model setting and some definitions.

\subsection{CMAB Problem Formulation}
A CMAB problem instance is modeled as a tuple $([m],\mathcal{I},D,R,Q)$.
	$[m] = \{1,2,\cdots,m\}$ is the set of base arms; $\mathcal{I} \subseteq 2^{[m]}$ is the set of super arms;
	 $D$ is a probability distribution in $[0,1]^m$, and is unknown to the player, $R$ and $Q$ are reward and
	 feedback functions to be specified shortly. Let $\bm{\mu} = (\mu_1,\cdots,\mu_m)$, where $\mu_i = \E_{\bm{X}\sim D}[X_i]$.
At discrete time slot $t \ge 1$, the player pulls a super arm $S(t) \in \mathcal{I}$, and the environment
	draws a random outcome vector $\bm{X}(t) = \{X_1(t),\cdots,X_m(t)\} \in [0,1]^m$ from $D$, independent of any other random variables.
Then the player receives an unknown reward $R(t) = R(S(t),\bm{X}(t))$, and observes
	the feedback $Q(t) = Q(S(t),\bm{X}(t))$.
As in \cite{chen2016combinatorial} and other papers studying CMAB, we consider semi-bandit feedback, that is,
	$Q(t) = \{(i,X_i(t))\mid i \in S(t)\}$.
At time $t$, the historical information is $\mathcal{F}_{t-1} = \{(S(\tau),Q(\tau)): 1 \le \tau \le t-1\}$,
	which is the input to the learning algorithm to select the action $S(t)$.
Similar to \cite{chen2016combinatorial}, we make the following two assumptions.
For a parameter vector $\bm{\mu}$, we use $\bm{\mu}_S$ to denote the projection of $\bm{\mu}$ on $S$, where $S\subseteq [m]$ is a subset of the base arms.


\begin{assumption}\label{Assumption2}
	The expected reward of a super arm $S\in \mathcal{I}$ only depends on the mean outcomes of base arms in $S$.
	That is, there exists a function $r$ such that $\E[R(t)] = \E_{\bm{X}(t) \sim D}[R(S(t),\bm{X}(t))] = r(S(t), \bm{\mu}_{S(t)})$.
\end{assumption}


The second assumption is a Lipschitz-continuity assumption of function $r$ to deal with non-linear reward functions (it is based on one-norm).
For a vector $\bm{x}=(x_1, \ldots, x_m)$, we use $||\bm{x}||_1 = \sum_{i\in [m]} |x_i|$ to denote the one-norm of $\bm{x}$.
\begin{assumption}\label{Assumption_Continouos1}
There exists a constant $B$, such that for every super arm $S$ and every pair of mean vectors $\bm{\mu}$ and $\bm{\mu'}$,
$ |r(S,\bm{\mu}) - r(S,\bm{\mu'})| \le B||\bm{\mu}_S-\bm{\mu'}_S||_1$.
\end{assumption}

%
The goal of the player is to minimize the total (expected) regret under time horizon $T$, as defined below:
\begin{equation*}
  Reg(T) \triangleq \E\left[\sum_{t=1}^T ( r(S^*,\bm{\mu}) - r(S(t),\bm{\mu}))\right],
\end{equation*}
%
%
%
where $S^* \in \argmax_{S \in \mathcal{I}} r(S,\bm{\mu})$ is a best super arm.

\subsection{Matroid Bandit}

In matroid bandit setting, $([m], \mathcal{I})$ is a matroid, which means that $\mathcal{I}\subseteq 2^{[m]}$ has two properties:
\begin{itemize}
  \item If $A \in \mathcal{I}$, then $\forall A' \subseteq A$, $A' \in \mathcal{I}$;
  \item If $A_1,A_2\in\mathcal{I}$, $|A_1| > |A_2|$, then there exists $i \in A_1\setminus A_2$ such that $A_2 \cup \{i\} \in \mathcal{I}$.
\end{itemize}
The reward function in matroid bandit is $R(S,\bm{x}) = \sum_{i\in S} x_i$, and thus the expected reward function is 
	$r(S,\bm{\mu}) = \sum_{i\in S} \mu_i$.

\section{Combinatorial Thompson Sampling} \label{sec:CTS}

We first consider the general CMAB setting.
For this setting, we assume that the player has an exact
	oracle ${\sf Oracle}(\bm{\theta})$ that takes a vector of parameters $\bm{\theta} = (\theta_1,\ldots,\theta_m)$ as input, and output a super arm
	$S = \argmax_{S \in \mathcal{I}} r(S,\bm{\theta})$.

The combinatorial Thompson sampling (CTS) algorithm is described in Algorithm \ref{Algorithm_TS}.
Initially we set the prior distribution of the means of all base arms as the Beta distribution $\beta(1,1)$,
	which is the uniform distribution on $[0,1]$.
After we get observation $Q(t)$, we update the prior distributions of all base arms in $S(t)$ using procedure
	{\sf Update} (Algorithm \ref{UP}):
	for each observation $X_i(t)$, we generate a Bernoulli random variable $Y_i(t)$ (the value of $Y_i$
	at time $t$) independently with mean $X_i(t)$, and then we update the prior Beta distribution of base arm $i$ using $Y_i(t)$ as the new observation.
It is easy to see that $\{Y_i(t)\}_t$'s
	are i.i.d.\ with the same mean $\mu_i$ as the samples $\{X_i(t)\}_t$'s., therefore we can use concentration inequalities (e.g., Chernoff-Hoeffding inequality) to analyze the empirical means of $\{Y_i(t)\}_t$'s.
Let $a_i(t)$ and $b_i(t)$ denote the values of $a_i$ and $b_i$ at the beginning of time step $t$ (here $a_i(t) - 1$ and $b_i(t) - 1$ represent the number of 1s and 0s in $\{Y_i(\tau)\}_{i \in S(\tau), \tau < t}$, respectively).
Then, following the Bayes' rule, the posterior distribution of parameter $\mu_i$ after observation $Q(t)$
	is $\beta(a_i(t) + Y_i(t), b_i(t) + 1- Y_i(t))$, which is what the {\sf Update} procedure does for
	updating $a_i$ and $b_i$.
When choosing a super arm, we simply draw independent samples from all base arms' prior distributions, i.e. $\theta_i(t) \sim \beta(a_i(t),b_i(t))$, and then send the sample vector $\bm{\theta}(t) = (\theta_1(t), \ldots, \theta_m(t))$
	to the oracle.
We use the output from the oracle $S(t)$ as the super arm to play.

We also need a further assumption to tackle the problem:
%
%
\begin{assumption}\label{Assumption_IND}
  $D = D_1 \times D_2 \times \cdots \times D_m$, i.e., the outcomes of all base arms are mutually independent.
\end{assumption}

This assumption is not necessary in the UCB-based CUCB algorithm~\citep{chen2016combinatorial}. However, when using the TS method, this assumption is needed. This is because that we are using the Bayes' rule, thus we need the exact likelihood function (as we can see in \cite{GMM14}). Only when the distributions for all the base arms are independent, we can use the {\sf Update} procedure (Algorithm \ref{UP}) to update their mean vector's prior distribution. When the distributions are correlated, the overall likelihood function does not equal to the production of likelihood functions of all the base arms, and this makes the update procedure too complicated to implement or analyze.

\begin{algorithm}[t]
    \centering
    \caption{Combinatorial Thompson Sampling (CTS) Algorithm for CMAB}\label{Algorithm_TS}
    \begin{algorithmic}[1]
    \STATE For each arm $i$, let $a_i = b_i = 1$
    \FOR {$t=1,2,\cdots$}
    \STATE For all arm $i$, draw a sample $\theta_i(t)$ from Beta distribution $\beta(a_i,b_i)$;
	    let $\bm{\theta}(t) = (\theta_1(t), \ldots, \theta_m(t))$
    \STATE Play action $S(t) = {\sf Oracle}(\bm{\theta}(t))$, get the observation $Q(t) = \{(i, X_i(t)): i\in S(t)\}$
    \STATE {\sf Update}$(\{(a_i,b_i) \mid i\in S(t)\}, Q(t))$
    \ENDFOR
    \end{algorithmic}
\end{algorithm}

\begin{algorithm}[t]
    \centering
    \caption{Procedure {\sf Update}}\label{UP}
    \begin{algorithmic}[1]
    \STATE \textbf{Input:} $\{(a_i,b_i)\mid  i\in S\}$, $Q = \{(i, X_i)\mid  i\in S\}$
    \STATE \textbf{Output:} updated $\{(a_i,b_i)\mid  i\in S\}$
    \FORALL{$(i, X_i) \in Q$}
    \STATE $Y_i \gets 1$ with probability $X_i$, $0$ with probability $1-X_i$
    \STATE $a_i \gets a_i+ Y_i$; $b_i \gets b_i+1 - Y_i$
    \ENDFOR
    \end{algorithmic}
\end{algorithm}

\subsection{Regret Upper Bound of CTS}
Let ${\sf OPT} = \argmax_{S \in \mathcal{I}} r(S,\bm{\mu})$ be the set of optimal super arms, and $S^* \in \argmin_{S\in {\sf OPT}} |S|$ is one of the optimal super arm with minimum size $k^*$. Then we can define $\Delta_{S} = r(S^*,\bm{\mu}) -  r(S,\bm{\mu})$ for all sub-optimal super arms $S$, and $\Delta_{\max} = \max_{S \in \mathcal{I}} \Delta_{S}, \Delta_{\min} = \min_{S \in \mathcal{I}} \Delta_{S}$. $K_{\max}$ is the maximum size of super arms, i.e. $K_{\max} = \max_{S\in \mathcal{I}} |S|$.

\begin{restatable}{theorem}{TheoremThree}\label{theorem_1}
  Under Assumptions~\ref{Assumption2}, \ref{Assumption_Continouos1}, and \ref{Assumption_IND}, for all $D$, Algorithm \ref{Algorithm_TS} has regret upper bound
  \begin{equation}\label{E1-6}
\sum_i {\left(2\log{K_{max}} +6 \right)B^2\log (2^m|\mathcal{I}|T)\over \min_{S: i\in S}(\Delta_{S} -({k^*}^2+2)B\varepsilon)} + \alpha_1 {8\Delta_{\max}\over \varepsilon^2}({4\over \varepsilon^2} + 1)^{k^*}\log{k^*\over \varepsilon^2} + (3m + {mK_{\max}^2\over \epsilon^2})\Delta_{\max}.
  \end{equation}
for any $\varepsilon$ such that $\forall S$ with $\Delta_S > 0$, $\Delta_{S} > 2B({k^*}^2+2)\varepsilon$, where $B$ is the Lipschitz
	constant in Assumption~\ref{Assumption_Continouos1}, and $\alpha_1$ is a constant
	not dependent on the problem instance. 

\end{restatable}

When $\varepsilon$ is sufficiently small, the leading $\log T$ term in the regret bound is 
\begin{eqnarray*}
&&\sum_i {\left(2\log{K_{max}} +6 \right)B^2\log (2^m|\mathcal{I}|T)\over \min_{S: i\in S}(\Delta_{S} -({k^*}^2+2)B\varepsilon)} \\
&=& \sum_{i=1}^m {\left(2\log{K_{max}} +6 \right)B^2\log T\over \min_{S: i\in S}(\Delta_{S} -({k^*}^2+2)B\varepsilon)} + O\left({m\left(2\log{K_{max}} +6 \right)B^2\log (2^m|\mathcal{I}|)\over \min_{S: i\in S}(\Delta_{S} -({k^*}^2+2)B\varepsilon)}\right).
\end{eqnarray*}
Since $|\mathcal{I}|$ is the number of super arms, which is upper bounded by $2^m$, we know the second term is always upper bounded by $O\left({B^2m^2\log K_{\max}\over \Delta_{\min} - (k^*+3)\varepsilon}\right)$ and does not depend on $T$. Thus, the $\log T$ term in Eq. \eqref{E1-6} is $O(m\log  K_{\max}\log T / \Delta_{\min})$, which is $\log K_{\max}$ better than the regret upper bound of ESCB in 
\citep{degenne2016combinatorial} and the regret upper bound of CTS in \cite{perrault2020statistical}. 
In fact, our novel theoretical analysis method can be used to reduce the regret upper bound of ESCB to $O(m\log  K_{\max}\log T / \Delta_{\min})$ as well.
Besides, in the ESCB policy, the player must compute the upper confidence bound for every super arm independently, which leads to an exponential time cost. In our CTS policy, we can use an efficient oracle to avoid the exponential cost. For example, if we want to find a path with minimum weight in a graph, CTS can use the Dijkstra's algorithm to find out the pulling path in each time slot (given the set of parameters $\bm{\theta}(t) = [\theta_1(t),\cdots,\theta_m(t)]$), but ESCB needs to list all the paths, and compute the lower confidence bounds for all of them. Thus, CTS can be more efficient, which makes it more attractive in applications. 

The term related to $\varepsilon$ in Eq. \eqref{E1-6} is to handle continuous Beta prior --- since we will never be able to sample a $\theta_i(t)$
	to be exactly the true value $\mu_i$, we need to consider the $\varepsilon$ neighborhood of $\mu_i$. 
This $\varepsilon$ term is common in most Thompson sampling analysis.
The exponential dependency on $k^*$ is because we need all the $k^*$ base arms in the best super arm $S^*$ to have 
	samples close to their means to make sure that it is the best super arm in sampling.
In contrast, for the top-$k$ MAB of \cite{komiyama2015optimal}, there is no such
	exponential dependency, because they only compare one base arm at a time (this can also be seen in our matroid bandit analysis).
When dealing with the general actions, the regret result of \cite{GMM14} also contains an exponentially
	large constant term without a close form, which is likely to be much larger than ours.
In Section~\ref{sec:exponentialconstant}, we show that this exponential constant
	is unavoidable for the general CTS.
	
\subsubsection{Proof of Theorem \ref{theorem_1}}\label{Section_Proof_T1}

We now provide the main proof of Theorem \ref{theorem_1}. Chernoff-Hoeffding inequality is very useful in the analysis, and thus we state it at the beginning.

\begin{fact}[Chernoff-Hoeffding inequalities \citep{chernoff1952measure,hoeffding1963probability}] \label{Fact_Chernoff}
When $X_1, X_2,\cdots,X_n$ are identical independent random variables such that $X_i \in [0,1]$ and $\E[X_i] = \mu$, we have the following inequalities:
\begin{equation*}
\Pr\left[{1\over n} \sum_{i=1}^n X_i \ge \mu + \epsilon \right] \le \exp(-2\epsilon^2 n),
\end{equation*}
 \begin{equation*}
\Pr\left[{1\over n} \sum_{i=1}^n X_i \le \mu - \epsilon \right] \le \exp(-2\epsilon^2 n).
\end{equation*}

\end{fact}

In this paper, we use $a_i(t)$ and $b_i(t)$ to denote the value of $a_i$ and $b_i$ at the beginning of time $t$. $\hat{\mu}_i(t) = {a_i(t)-1 \over N_i(t)} = {1\over N_i(t)}\sum_{\tau: \tau < t, i\in S(\tau)} Y_i(t)$ is the empirical mean of arm $i$ at the beginning of time $t$, where $N_i(t) = a_i(t)+ b_i(t)-2$ is the number of observations of arm $i$ at the beginning of time $t$.
 Notice that for fixed arm $i$, in different time $t$ with $i\in S(t)$, $X_i(t)$'s are i.i.d. with mean $\mu_i$, and $Y_i(t)$ is a Bernoulli random variable with mean $X_i(t)$, thus the Bernoulli random variables $Y_i(t)$'s are also i.i.d. with mean $\mu_i$. Based on $\hat{\mu}_i(t)$, we can define the following five events:

\begin{itemize}
	\item $\mathcal{A}(t) = \{S(t)\notin {\sf OPT}\}$;
	\item $\mathcal{B}(t) = \{\exists i \in S(t), |\hat{\mu}_i(t) - \mu_i| > {\varepsilon\over |S(t)|}\}$;
	\item $\mathcal{C}(t) = \{||\bm{\theta}_{S(t)}(t) - \bm{\mu}_{S(t)}||_1 > {\Delta_{S(t)}\over B} - ({k^*}^2+1)\varepsilon\}$;
	\item $\mathcal{G}(t) = \{\sum_{i\in S(t)} |\theta_{i}(t) -\hat{\mu}_i(t)| > {\Delta_{S(t)}\over B} - ({k^*}^2+2)\varepsilon\}$;
	\item $\mathcal{H}(t) = \{\sum_{i\in S(t)} {1\over N_i(t)} \le {2({\Delta_{S(t)}\over B}- ({k^*}^2+2) \varepsilon)^2 \over \log (2^m|\mathcal{I}|T)}\}$.
	
  \end{itemize}


Then the total regret can be written as:
\begin{eqnarray*}
\sum_{t=1}^T\E\left[\I[\mathcal{A}(t)] \times \Delta_{S(t)}\right]
&\le& \sum_{t=1}^T\E\left[\I[\mathcal{B}(t) \land \mathcal{A}(t)] \times \Delta_{S(t)}\right] \\
&& + \sum_{t=1}^T\E\left[\I[\neg \mathcal{B}(t) \land \mathcal{C}(t) \land \mathcal{A}(t)] \times \Delta_{S(t)}\right] \\
&& + \sum_{t=1}^T\E\left[\I[\neg \mathcal{C}(t)\land \mathcal{A}(t)] \times \Delta_{S(t)}\right].
\end{eqnarray*}

\noindent\textbf{The first term:}

We can use the following lemma to bound the first term. 
\begin{lemma}\label{Lemma_M1}
In Algorithm \ref{Algorithm_TS}, we have
\begin{equation*}
  \E[|t:1\le t \le T, i\in S(t), |\hat{\mu}_i(t) - \mu_i| > \varepsilon|] \le 1+{1\over \varepsilon^2}
\end{equation*}
for any $1 \le i \le m$.
\end{lemma}
\begin{proof} Let $\tau_1,\tau_2,\cdots$ be the time slots such that $i\in S(t)$ and define $\tau_0 = 0$, then 
  \begin{eqnarray}
    \nonumber&&\E[|t: 1\le t \le T; i\in S(t),|\hat{\mu}_i(t) - \mu_i| > \varepsilon|]\\ 
    \nonumber=&& \E\left[\sum_{t=1}^T \I[i\in S(t), |\hat{\mu}_i(t) - \mu_i| > \varepsilon]\right] \\
    \nonumber\le&& \E\left[\sum_{w=0}^T \E\left[\sum_{t=\tau_w}^{\tau_{w+1} - 1} \I[i\in S(t), |\hat{\mu}_i(t) - \mu_i| > \varepsilon]\right] \right]\\
    \nonumber\le&& \E\left[\sum_{w=0}^T \Pr[|\hat{\mu}_i(t) - \mu_i| > \varepsilon, N_i = w] \right]\\
    \nonumber\le&& 1+\sum_{w=1}^T \Pr[|\hat{\mu}_i(t) - \mu_i| > \varepsilon, N_i = w]  \\
    \label{Eq_8901}\le&& 1+\sum_{w=1}^T \exp(-2w \varepsilon^2) + \sum_{w=1}^T \exp(-2w \varepsilon^2)\\
    \nonumber\le&& 1+ 2\sum_{w=1}^{\infty} (\exp(-2\varepsilon^2))^w\\
    \nonumber\le&& 1+2\cdot {\exp(-2\varepsilon^2) \over 1 - \exp(-2\varepsilon^2)}  \\
   \nonumber\le&& 1+{2\over 2\varepsilon^2}\\
   \nonumber=&& 1+{1\over \varepsilon^2},
  \end{eqnarray}
where Eq. \eqref{Eq_8901} is because of the Chernoff-Hoeffding inequality (Fact \ref{Fact_Chernoff}). 
\end{proof}

By Lemma \ref{Lemma_M1}, we know that the first term is upper bounded by $ \left({mK_{\max}^2\over \varepsilon^2}+m\right)\Delta_{\max}$. 

\noindent\textbf{The second term:}

Now we come to the second term. For this regret term, we improve the upper bound from $O(m(\log K_{\max})^2\log T/\Delta_{\min})$ (in prior works \cite{degenne2016combinatorial,perrault2020statistical}) to $O(m\log K_{\max}\log T/\Delta_{\min})$, and \emph{this is one of the main contributions of this manuscript}.  

Notice that under $\neg \mathcal{B}(t) \land \mathcal{C}(t) $, we must have that $\sum_{i\in S(t)} |\theta_{i}(t) -\hat{\mu}_i(t)| > {\Delta_{S(t)}\over B} - ({k^*}^2+2)\varepsilon$, i.e., event $\mathcal{G}(t)$ must happen.

Then the second term can be bounded by 
\begin{eqnarray*}
\sum_{t=1}^T\E\left[\I[\neg \mathcal{B}(t) \land \mathcal{C}(t) \land\mathcal{A}(t)] \times \Delta_{S(t)}\right]&\le& \sum_{t=1}^T\E\left[\I[\mathcal{G}(t) \land \mathcal{H}(t) \land\mathcal{A}(t)] \times \Delta_{S(t)}\right] \\
&& + \sum_{t=1}^T\E\left[\I[\mathcal{G}(t) \land \neg \mathcal{H}(t) \land\mathcal{A}(t)] \times \Delta_{S(t)}\right]. 
\end{eqnarray*}

The following fact in \cite{perrault2020statistical} shows that $\sum_{t=1}^T\E\left[\I[\mathcal{G}(t) \land \mathcal{H}(t) \land\mathcal{A}(t)] \times \Delta_{S(t)}\right]$ is upper bounded by $\Delta_{\max}$.

\begin{fact}[Lemma 2 in \cite{perrault2020statistical}]\label{Fact_Subgaussian}
For any $t>0$, we have that
\begin{equation*}
\Pr[\mathcal{G}(t) \land \mathcal{H}(t)] \le {1\over T}.
\end{equation*}
\end{fact}

Now we come to bound the regret term $\sum_{t=1}^T \E[\I[\mathcal{G}(t) \land \neg \mathcal{H}(t) \land \mathcal{A}(t)]\times \Delta_{S(t)}]$. Here we use a regret allocation method to count this regret term. That is, for any time step $t$ such that $\mathcal{G}(t) \land \neg \mathcal{H}(t) \land \mathcal{A}(t)$ happens, we allocate regret $g_i(N_i(t))$ to each base arm $i\in S(t)$. We say the allocation functions $g_i$'s are correct if the sum of allocated regrets in this step is larger than $\Delta_{S(t)}$, i.e., $\sum_{i\in S(t)} g_i(N_i(t)) \ge \Delta_{S(t)}$. Note that under event $\mathcal{G}(t) \land \neg \mathcal{H}(t) \land \mathcal{A}(t)$, for all the base arms $i\in S(t)$, $N_i(t+1) = N_i(t) + 1$. Therefore, for each base arm $i$, the allocated regret $g_i(n)$ appears at most once for any $n\ge 0$. If we have correct allocation functions $g_i$'s, then the regret term $\sum_{t=1}^T \E[\I[\mathcal{G}(t) \land \neg \mathcal{H}(t) \land \mathcal{A}(t)]\times \Delta_{S(t)}]$ can be upper bounded by $\sum_i \sum_{n=0}^\infty g_i(n)$.

Then we describe our allocation functions $g_i$'s. Here we define $L_{i,1} =  {K_{\max}\log (2^m|\mathcal{I}|T)\over \min_{S: i\in S} ({\Delta_{S}\over B} -({k^*}^2+2)\varepsilon)^2}$ and $L_{i,2} = {\log (2^m|\mathcal{I}|T)\over \min_{S: i\in S} ({\Delta_{S}\over B} -({k^*}^2+2)\varepsilon)^2}$, and $g_i(n)$ is defined as follows:
\begin{equation*}
  g_i(n) =  \begin{cases}
    \Delta_{\max}, & n = 0\\
   2B\sqrt{\log (2^m|\mathcal{I}|T) \over n},  & 0 < n \le L_{i,2} \\
{2B\log (2^m|\mathcal{I}|T)\over n\min_{S: i\in S}({\Delta_{S}\over B} -({k^*}^2+2)\varepsilon)}, &  L_{i,2} < n \le L_{i,1}\\
0, & n > L_{i,1}
  \end{cases}
\end{equation*}

Now we prove that these allocation functions $g_i$'s satisfy the correctness condition when $\varepsilon \le {\Delta_{\min}\over 2B({k^*}^2+2)}$, i.e., if event $\mathcal{G}(t) \land \neg \mathcal{H}(t) \land \mathcal{A}(t)$ happens, then $\sum_{i\in S(t)} g_i(N_i(t)) \ge \Delta_{S(t)}$.

If there exists $i\in S(t)$ such that $N_i(t) = 0$, then $g_i(N_i(t)) = \Delta_{\max} \ge \Delta_{S(t)}$. Since $g_i(n)$ is always non-negative, we know that $\sum_{i\in S(t)} g_i(N_i(t)) \ge \Delta_{S(t)}$.

If there exists $i\in S(t)$ such that $1\le N_i(t) \le {\log (2^m|\mathcal{I}|T) \over  ({\Delta_{S(t)}\over B} -({k^*}^2+2)\varepsilon)^2}$, then $N_i(t) \le L_{i,2}$ and therefore 
\begin{eqnarray*}
g_i(N_i(t)) &=& 2B\sqrt{\log (2^m|\mathcal{I}|T) \over N_i(t)} \\
&\ge& 2B\sqrt{\log (2^m|\mathcal{I}|T) \over {\log (2^m|\mathcal{I}|T) \over  ({\Delta_{S(t)}\over B} -({k^*}^2+2)\varepsilon)^2}} \\
&=& 2B \left({\Delta_{S(t)}\over B} -({k^*}^2+2)\varepsilon\right)\\
&\ge& \Delta_{S(t)},
\end{eqnarray*}
where the last inequality is because that $\varepsilon \le {\Delta_{\min}\over 2B({k^*}^2+2)}$. From the above inequalities, we know that $\sum_{i\in S(t)} g_i(N_i(t)) \ge \Delta_{S(t)}$.

If for all $i\in S(t)$, $N_i(t) > {\log (2^m|\mathcal{I}|T) \over  ({\Delta_{S(t)}\over B} -({k^*}^2+2)\varepsilon)^2}$, then we use $S_1(t)$ to denote the set of arms $i\in S(t)$ such that $N_i(t) > L_{i,1}$, $S_2(t)$ to denote the set of arms $i\in S(t)$ such that $L_{i,2} < N_i(t) \le L_{i,1}$, $S_3(t)$ to denote the set of arms $i\in S(t)$ such that $N_i(t) \le L_{i,2}$, by definition of allocation functions $g_i$'s, we have that
\begin{eqnarray}
\nonumber&&\sum_{i\in S(t)} g_i(N_i(t))\\
\nonumber&=&\sum_{i\in S_3(t)} 2B\sqrt{\log (2^m|\mathcal{I}|T) \over N_i(t)} + \sum_{i\in S_2(t)} {2B\log (2^m|\mathcal{I}|T)\over N_i(t)\min_{S: i\in S}({\Delta_{S}\over B} -({k^*}^2+2)\varepsilon)}\\ 
\nonumber&\ge& \sum_{i\in S_3(t)} 2B\sqrt{\log (2^m|\mathcal{I}|T) \over N_i(t)} + \sum_{i\in S_2(t)} {2B\log (2^m|\mathcal{I}|T)\over N_i(t)({\Delta_{S(t)}\over B} -({k^*}^2+2)\varepsilon)}\\
\nonumber&=& \sum_{i\in S_3(t)} {2B\log (2^m|\mathcal{I}|T)\over N_i(t)({\Delta_{S(t)}\over B} -({k^*}^2+2)\varepsilon)} \cdot \sqrt{ N_i(t)({\Delta_{S(t)}\over B} -({k^*}^2+2)\varepsilon)^2\over \log (2^m|\mathcal{I}|T)} \\
\nonumber&&+ \sum_{i\in S_2(t)} {2B\log (2^m|\mathcal{I}|T)\over N_i(t)({\Delta_{S(t)}\over B} -({k^*}^2+2)\varepsilon)}\\
\label{Eq---1221}&\ge& \sum_{i\in S_3(t)}  {2B\log (2^m|\mathcal{I}|T)\over N_i(t)({\Delta_{S(t)}\over B} -({k^*}^2+2)\varepsilon)} + \sum_{i\in S_2(t)} {2B\log (2^m|\mathcal{I}|T)\over N_i(t)({\Delta_{S(t)}\over B} -({k^*}^2+2)\varepsilon)}\\
\nonumber&=& \sum_{i\in S(t) \setminus S_1(t)}  {2B\log (2^m|\mathcal{I}|T)\over N_i(t)({\Delta_{S(t)}\over B} -({k^*}^2+2)\varepsilon)} \\
\nonumber&=& {2B\log (2^m|\mathcal{I}|T)\over ({\Delta_{S(t)}\over B} -({k^*}^2+2)\varepsilon)} \left(\sum_{i\in S(t)}  {1\over N_i(t)} - \sum_{i\in S_1(t)}  {1\over N_i(t)}\right) \\
\label{Eq---1222}&\ge& {2B\log (2^m|\mathcal{I}|T)\over ({\Delta_{S(t)}\over B} -({k^*}^2+2)\varepsilon)} \left( {2({\Delta_{S(t)}\over B} -({k^*}^2+2)\varepsilon)^2 \over \log (2^m|\mathcal{I}|T)}
- \sum_{i\in S_1(t)}  {\min_{S: i\in S} ({\Delta_{S}\over B} -({k^*}^2+2)\varepsilon)^2\over K_{\max}\log (2^m|\mathcal{I}|T)}\right)\\
\nonumber&\ge& {2B\log (2^m|\mathcal{I}|T)\over ({\Delta_{S(t)}\over B} -({k^*}^2+2)\varepsilon)} \left( {2({\Delta_{S(t)}\over B} -({k^*}^2+2)\varepsilon)^2 \over \log (2^m|\mathcal{I}|T)}
- \sum_{i\in S_1(t)}  {({\Delta_{S(t)}\over B} -({k^*}^2+2)\varepsilon)^2\over K_{\max}\log (2^m|\mathcal{I}|T)}\right)\\
\nonumber&\ge& {2B\log (2^m|\mathcal{I}|T)\over ({\Delta_{S(t)}\over B} -({k^*}^2+2)\varepsilon)} \left( {2({\Delta_{S(t)}\over B} -({k^*}^2+2)\varepsilon)^2 \over \log (2^m|\mathcal{I}|T)}
- K_{\max}\cdot {({\Delta_{S(t)}\over B} -({k^*}^2+2)\varepsilon)^2\over K_{\max}\log (2^m|\mathcal{I}|T)}\right)\\
\nonumber&=& {2B\log (2^m|\mathcal{I}|T)\over ({\Delta_{S(t)}\over B} -({k^*}^2+2)\varepsilon)} \cdot {({\Delta_{S(t)}\over B} -({k^*}^2+2)\varepsilon)^2 \over \log (2^m|\mathcal{I}|T)}\\
\nonumber&=& 2B\left({\Delta_{S(t)}\over B} -({k^*}^2+2)\varepsilon\right)^2\\
\nonumber&\ge& \Delta_{S(t)}.
\end{eqnarray}
Here Eq. \eqref{Eq---1221} is because that $N_i(t) > {\log (2^m|\mathcal{I}|T) \over  ({\Delta_{S(t)}\over B} -({k^*}^2+2)\varepsilon)^2}$ (as we assumed in the beginning of this paragraph), Eq. \eqref{Eq---1222} comes from the definition of $\neg \mathcal{H}(t)$ (the first term) and the definition of $S_1(t)$ (the second term). And this finishes the proof that allocation functions $g_i$'s satisfy the correctness condition when $\varepsilon \le {\Delta_{\min}\over 2B({k^*}^2+2)}$.

Because of this, the second regret term is upper bounded by
\begin{align}
\nonumber&\sum_{t=1}^T\E\left[\I[\neg \mathcal{B}(t) \land \mathcal{C}(t) \land \mathcal{A}(t)] \times\Delta_{S(t)}\right] \\
\nonumber &\le (m+1)\Delta_{\max} + \sum_{i} \sum_{n=1}^{L_{i,2}} 2B\sqrt{\log (2^m|\mathcal{I}|T)\over n} + \sum_{i} \sum_{n=L_{i,2}+1}^{L_{i,1}} {1\over n} \cdot {2B\log (2^m|\mathcal{I}|T)\over \min_{S: i\in S} ({\Delta_{S}\over B} -({k^*}^2+2)\varepsilon)}\\
\label{Eq_m101-2}&\le (m+1)\Delta_{\max} + \sum_{i} 4B\sqrt{\log (|\mathcal{I}|T) L_{i,2}} + \sum_{i} \left(1+\log{L_{i,1} \over L_{i,2}}\right)\cdot {2B\log (2^m|\mathcal{I}|T)\over \min_{S: i\in S} ({\Delta_{S}\over B} -({k^*}^2+2)\varepsilon)}.
\end{align} 
Here Eq. \eqref{Eq_m101-2} is because that
 $\sum_{n=1}^N \sqrt{1/n} \le 2\sqrt{N}$ (by a simple inductive proof on $N$)
 and $\sum_{n=N_1}^{N_2} {1\over n} \le 1+\log{N_2 \over N_1}$.
 
The value $\sqrt{\log (2^m|\mathcal{I}|T) L_{i,2}}$ equals to ${\log (2^m|\mathcal{I}|T)\over \min_{S: i\in S}({\Delta_{S}\over B} -({k^*}^2+2)\varepsilon)}$, and $\log{L_{i,1} \over L_{i,2}} = \log{K_{max}}$, therefore, the total regret in the second term is upper bounded by
\begin{equation*}
2m\Delta_{\max} + \sum_i \left(2\log{K_{max}} +6 \right){B^2\log (2^m|\mathcal{I}|T)\over \min_{S: i\in S}(\Delta_{S} -({k^*}^2+2)B\varepsilon)}.
\end{equation*}

\begin{remark}
Note that the above analysis implies that using $g'_i(n) = {2B\log (2^m|\mathcal{I}|T)\over n\min_{S: i\in S}({\Delta_{S}\over B} -({k^*}^2+2)\varepsilon)}$ for any $n \le L_{i,1}$ is sufficient to satisfy the correctness condition. However, these allocation functions are too large when $n$ is small. Therefore we use piecewise functions $g_i$'s to count the cumulative regret instead.
\end{remark}

\begin{remark}
The analysis in \cite{degenne2016combinatorial, perrault2020statistical} has the similar idea with using the following allocation functions $g''_i$'s to count the regret (although their counting method is very different with ours). Define $L''_{i,2} = {B^2\log K_{\max}\log T\over \min_{S: i\in S} \Delta_{S}^2}$, then $g''(n) \approx \min_{S: i\in S} \Delta_{S}$ for $n < L''_{i,2}$ and $g''(n)\approx {B^2\log K_{\max}\log T\over n\min_{S: i\in S} \Delta_{S}}$ for $L''_{i,2} < n \le K_{\max}L''_{i,2}$. We can see that this leads to $O(m(\log K_{\max})^2\log T/\Delta_{\min})$ regret upper bound.
%
Compare to their analysis, ours reduces one $\log K_{\max}$ factor in both $L''_{i,2}$ and $g''_i$ (i.e., our $L_{i,2}$ is $\log K_{\max}$ smaller than $L''_{i,2}$ and our $g_i$ is also $\log K_{\max}$ smaller than $g''_i$), thus it also reduces one $\log K_{\max}$ factor in the regret upper bound.  
\end{remark}

\noindent\textbf{The third term:}

The difficulty in the analysis of the third term mainly lies in the intersections between super arms, and how we deal with it is one of the main contributions in our conference paper.  Therefore, the proofs of Lemmas \ref{Lemma45} and \ref{Lemma3} (which are stated below) are referred to \ref{sec:proofTheorem1}.

Although the $\theta_i(t)$'s of all base arms are mutually independent,
	when it comes to super arms, the value $r(S,\bm{\theta}(t)_S)$'s for different super arms $S$
	are not mutually independent, because super arms may overlap one another.
For example, Lemma 1 in \cite{agrawal2013further} is not true for considering super arms because of the lack of independence.
This means that we cannot simply use the technique of \cite{agrawal2013further}.
In this paper, we deal with it in the following way.

Let $\bm{\theta} = (\theta_1,\cdots,\theta_m)$ be a vector of parameters, $Z \subseteq [m]$ and $Z \ne \emptyset$ be some base arm set and $Z^c$ be the complement of $Z$.
Recall that $\bm{\theta}_Z$ is the sub-vector of $\bm{\theta}$ projected onto $Z$, and
	we use notation $(\bm{\theta'}_Z, \bm{\theta}_{Z^c})$ to denote replacing $\theta_i$'s with $\theta'_i$'s for $i \in Z$ and keeping
	the values $\theta_i$ for $i\in Z^c$ unchanged.

Given a subset $Z \subseteq S^*$, we consider the following property for $\bm{\theta}_{Z^c}$. For any $\bm{\theta'}_Z$ such that $||\bm{\theta'}_Z - \bm{\mu}_Z||_\infty \le \varepsilon$, let $\bm{\theta'} = (\bm{\theta'}_Z, \bm{\theta}_{Z^c})$, then:
\begin{itemize}
  \item $Z \subseteq {\sf Oracle}(\bm{\theta'})$;
  \item Either ${\sf Oracle}(\bm{\theta'}) \in {\sf OPT}$ or $||\bm{\theta}_{{\sf Oracle}(\bm{\theta'})}' - \bm{\mu}_{{\sf Oracle}(\bm{\theta'})}||_1 > {\Delta_{{\sf Oracle}(\bm{\theta'})}\over B} - ({k^*}^2+1)\varepsilon $.
\end{itemize}

The first one is to make sure that if we have normal samples in $Z$ at time $t$ (i.e., the sample value $\theta_i(t)$ is within
	$\varepsilon$ neighborhood of  $\mu_i$ for all $i\in Z$), then all the arms in $Z$ will be played and observed.
These observations would update the Beta distributions of these base arms to be more accurate, such that the probability of the next time that the samples from these base arms are also within
	$\varepsilon$ neighborhood of their true mean value becomes larger.
This fact would be used later in the quantitative regret analysis.
The second one says that if the samples in $Z$ are normal,
	then $\neg \mathcal{C}(t)\land \mathcal{A}(t)$ can not happen (similar to the analysis in \cite{agrawal2013further} and \cite{komiyama2015optimal}). 
We use $\mathcal{E}_{Z,1}(\bm{\theta})$ to denote the event that the vector $\bm{\theta}_{Z^c}$ has such a property,
and emphasize that this event only depends on the values in vector $\bm{\theta}_{Z^c}$. 

What we want to do is to find some exact $Z$ such that $\mathcal{E}_{Z,1}(\bm{\theta}(t))$ happens when $\neg \mathcal{C}(t)\land \mathcal{A}(t)$ happens. 
If such $Z$ exists, then for any $t$ such that $\mathcal{E}_{Z,1}(\bm{\theta}(t))$ happens, there are two possible cases: i) the samples of all the arms $i\in Z$ are normal, which means $\neg \mathcal{C}(t)\land \mathcal{A}(t)$ cannot happen, and will update the posterior distributions of all the arms $i\in Z$ to increase the probability that the samples of all the arms $i\in Z$ are normal; ii) the samples of some arms $i\in Z$ are not normal, and $\neg \mathcal{C}(t)\land \mathcal{A}(t)$ may happen in this case. 
As time going on, the probability that the samples in $Z$ are normal becomes larger and larger, and therefore the probability that $\neg \mathcal{C}(t)\land \mathcal{A}(t)$ happens becomes smaller and smaller. Thus, $\sum_{t=1}^T \E[\I[\neg \mathcal{C}(t)\land \mathcal{A}(t)]]$ has a constant upper bound.

The following lemma shows that such $Z$ must exist, it is the key lemma in the analysis of the third term.

\begin{restatable}{lemma}{KeyLemma}\label{Lemma45}
  Suppose that $\neg \mathcal{C}(t)\land \mathcal{A}(t)$ happens, 
  then there exists $Z \subseteq S^*$ and $Z \ne \emptyset$ such that $\mathcal{E}_{Z,1}(\bm{\theta}(t))$ holds.
\end{restatable}

By Lemma~\ref{Lemma45}, for some nonempty $Z$, $\mathcal{E}_{Z,1}(\bm{\theta}(t))$ occurs when $\neg \mathcal{C}(t)\land \mathcal{A}(t)$ happens. Another fact is that $||\bm{\theta}_Z(t) - \bm{\mu}_Z||_\infty > \varepsilon$. The reason is that if $||\bm{\theta}_Z(t) - \bm{\mu}_Z||_\infty \le \varepsilon$, by definition of the property, either $S(t) \in {\sf OPT}$ or $||\bm{\theta}_{S(t)}(t) - \bm{\mu}_{S(t)}||_1 > {\Delta_{S(t)}\over B} - ({k^*}^2+1)\varepsilon$, which means $\neg \mathcal{C}(t)\land \mathcal{A}(t)$ can not happen.
Let $\mathcal{E}_{Z,2}(\bm{\theta})$ be the event $\{||\bm{\theta}_Z - \bm{\mu}_Z||_\infty > \varepsilon \}$.
Then $\{\neg \mathcal{C}(t)\land \mathcal{A}(t)\}\to \lor_{Z \subseteq S^*, Z \ne \emptyset} (\mathcal{E}_{Z,1}(\bm{\theta}(t)) \land \mathcal{E}_{Z,2}(\bm{\theta}(t)))$.

Using similar techniques in \cite{komiyama2015optimal}, we can get the upper bound $O\left({8\over \varepsilon^2}({4\over \varepsilon^2})^{|Z|}\log{|Z|\over \varepsilon^2} \right)$ for $\sum_{t=1}^T  \E\left[\I\{\mathcal{E}_{Z,1}(\bm{\theta}(t)),\mathcal{E}_{Z,2}(\bm{\theta}(t))\}\right]$, which is stated in the following lemma.

\begin{restatable}{lemma}{LemmaThree}\label{Lemma3}
  In Algorithm \ref{Algorithm_TS}, under Assumption \ref{Assumption_IND}, for all $Z \subset S^*$ and $Z \ne \emptyset$, 
  \begin{equation*}
    \sum_{t=1}^T \E[\I\{\mathcal{E}_{Z,1}(\bm{\theta}(t)),\mathcal{E}_{Z,2}(\bm{\theta}(t))\}] \le 13\alpha'_2 \cdot \left({2^{2|Z|+3}\log{|Z|\over \varepsilon^2} \over \varepsilon^{2|Z|+2}}\right).
  \end{equation*}
\end{restatable} 

From Lemma \ref{Lemma3}, we have
\begin{eqnarray*}
  \sum_{Z \subseteq S^*, Z \ne \emptyset} \left(\sum_{t=1}^T \E\left[\I\{\mathcal{E}_{Z,1}(\bm{\theta}(t)) ,\mathcal{E}_{Z,2}(\bm{\theta}(t))\}\right]\right) &\le& \sum_{Z \subseteq S^*, Z \ne \emptyset} 13\alpha'_2 \cdot \left({2^{2|Z|+3}\log{|Z|\over \varepsilon^2} \over \varepsilon^{2|Z|+2}}\right) \\
  &\le& 13\alpha'_2 {8\over \varepsilon^2}\log{k^*\over \varepsilon^2} \sum_{Z \subseteq S^*, Z \ne \emptyset}  {2^{2|Z|} \over \varepsilon^{2|Z|}}\\
  &\le& 13\alpha'_2 {8\over \varepsilon^2}({4\over \varepsilon^2} + 1)^{k^*}\log{k^*\over \varepsilon^2}.
\end{eqnarray*}

\noindent\textbf{Sum of all the terms:}

The regret upper bound of CTS is the sum of these three terms, i.e., 
\begin{equation*}
\sum_i {\left(2\log{K_{max}} +6 \right)B^2\log (2^m|\mathcal{I}|T)\over \min_{S: i\in S}(\Delta_{S} -({k^*}^2+2)B\varepsilon)} + 13\alpha'_2 {8\Delta_{\max}\over \varepsilon^2}({4\over \varepsilon^2} + 1)^{k^*}\log{k^*\over \varepsilon^2} + (3m + {mK_{\max}^2\over \epsilon^2})\Delta_{\max}.
\end{equation*}


Let $\alpha_1 = 13\alpha'_2$, we know it is a constant that does not depend on the problem instance.

%

\subsection{Approximation Oracle}

We consider using an approximation oracle in our CTS algorithm as well, like what the authors did in \citep{chen2016combinatorial} or \citep{wen2015efficient}. 
However, we found out that Thompson sampling does not directly work with an
	approximation oracle even in the original MAB model, as shown in Theorem \ref{Proposition1}.
Notice that here we do not consider the Bayesian regret, so it does not contradict with the results in \citep{wen2015efficient}.

To make it clear, we need to show the definitions of approximation oracle and approximation regret here.

\begin{definition}
  An approximation oracle with approximation ratio $\lambda$ for the MAB problem is a function ${\sf Oracle}:[0,1]^m \to \{1,\cdots,m\}$ such that for all $\bm{\mu}\in [0,1]^m$, $\mu_{\sf Oracle(\bm{\mu})} \ge \lambda\max_i\mu_i$.
\end{definition}

\begin{definition}
  The approximation regret with an approximation ratio $\lambda$ of an MAB algorithm on mean vector $\bm{\mu}$ is defined as:
  \begin{equation*}
    \sum_{t=1}^T (\lambda\max_i\mu_i - \mu_{i(t)}),
  \end{equation*}
  where $i(t)$ is the arm pulled by the algorithm at time step $t$.
\end{definition}

The TS algorithm using approximation oracle works the same as Algorithm \ref{Algorithm_TS}, where
	 ${\sf Oracle}$ now is the approximation oracle.

\begin{restatable}{theorem}{Theorem}\label{Proposition1}
There exist an MAB instance $\cM$ and an approximation oracle ${\sf Oracle}$ such that
	the approximation regret of Algorithm \ref{Algorithm_TS} with ${\sf Oracle}$ on $\cM$
	is $\Omega(T)$.
	
\end{restatable}

\subsubsection{Proof of Theorem \ref{Proposition1}}\label{Section_Proof_2}

Now we give the theoretical analysis of Theorem \ref{Proposition1}. \emph{This is another main contribution of this manuscript.}

Azuma's inequality, Markov's inequality and Beta-Binomial trick are very useful in the analysis, thus we state them at the beginning.

\begin{fact}[Azuma's inequality, sub-martingale case \citep{azuma1967weighted}] \label{Fact_Azuma}
 If $X_0 = 0, X_1,\cdots,X_n$ satisfies that $\E[X_i|X_0,X_1,\cdots X_{i-1}] \ge X_{i-1} + \mu$ for some constant $\mu$, and $|X_i - X_{i-1}| \le 1$, then for any $\epsilon > 0$, 
\begin{equation*}
\Pr[X_n \le n\mu-\epsilon] \le \exp\left(-{\epsilon^2 \over 2n}\right).
\end{equation*}

\end{fact}


\begin{fact}[Markov's inequality \citep{stein2003princeton}]  \label{Fact_Markov}
If $X$ is a non-negative random variable, then for any $\epsilon > 0$,
\begin{equation*}
\Pr[X\ge \epsilon] \le {\E[X] \over \epsilon}.
\end{equation*}

\end{fact}

\begin{fact}[Beta-Binomial trick, Fact 3 in \cite{agrawal2013further}]\label{Fact33}
	Let $F^{Beta}_{a,b}(x)$ be the CDF of Beta distribution with parameters $(a,b)$, let $F^B_{n,p}(x)$ be the CDF of Binomial distribution with parameters $(n,p)$. Then for any positive integers $(a,b)$, we have
	\begin{equation*}
		F^{Beta}_{a,b}(x) = 1 - F^B_{a+b-1,x}(a-1).
	\end{equation*}
\end{fact}

We also provide some useful lemmas (and their proofs) here. 

\begin{lemma}\label{Fact_M15}
In Algorithm \ref{Algorithm_TS}, for any base arm $i$, we have that:
\begin{equation*}
  \Pr\left[\theta_i(t)-\hat{\mu}_i(t)  > \epsilon | a_i(t), b_i(t)\right] \le \exp(-2{N_i(t)\epsilon^2}),
\end{equation*}
\begin{equation*}
  \Pr\left[\hat{\mu}_i(t) - \theta_i(t) > \epsilon | a_i(t), b_i(t)\right] \le \exp(-2{N_i(t)\epsilon^2}).
\end{equation*}
\end{lemma}

\begin{proof}
We only need to prove the first inequality, the second one can be done by symmetry.

Note that given $a_i(t), b_i(t)$, $N_i(t) = a_i(t) + b_i(t) - 2$ is also fixed. Therefore, 
\begin{eqnarray}
		\nonumber&&\Pr[\theta_i(t) > \hat{\mu}_i(t) + \epsilon| a_i(t), b_i(t)] \\
		\nonumber&=& 1 - F^{Beta}_{a_i(t),b_i(t)}(\hat{\mu}_i(t) + \epsilon)\\
		\label{eq_m100}&=& 1 - (1 - F^B_{a_i(t)+b_i(t)-1, \hat{\mu}_i(t) + \epsilon}(a_i(t) - 1))\\
		\nonumber&=& F^B_{a_i(t)+b_i(t)-1, \hat{\mu}_i(t) + \epsilon}(a_i(t) - 1)\\
         \label{eq_m301}&=& F^B_{a_i(t)+b_i(t)-1, \hat{\mu}_i(t) + \epsilon}(\hat{\mu}_i(t)(a_i(t)+b_i(t)-2))\\
		\nonumber&\le& F^B_{a_i(t)+b_i(t)-1, \hat{\mu}_i(t) + \epsilon}(\hat{\mu}_i(t)(a_i(t)+b_i(t)-1))\\
		\label{eq_m101}&\le&\exp(-2(a_i(t)+b_i(t)-1)\epsilon)\\
		\nonumber&\le& \exp(-2{N_i(t)\epsilon^2}).
	\end{eqnarray}
	
Eq. \eqref{eq_m100} is given by Beta-Binomial Trick (Fact \ref{Fact33}), Eq. \eqref{eq_m301} is given by the definition $\hat{\mu}_i(t) = {a_i(t) -1 \over a_i(t)+b_i(t) - 2}$, Eq. \eqref{eq_m101} is given by Chernoff-Hoeffding inequality (Fact \ref{Fact_Chernoff}).
\end{proof}

\begin{lemma}\label{Fact_M5}
In Algorithm \ref{Algorithm_TS}, for any base arm $i$, we have the following two inequalities:
\begin{equation*}
  \Pr\left[\theta_i(t)-\hat{\mu}_i(t)  > \sqrt{4\log t\over N_i(t)}\right] \le {1\over t^2},
\end{equation*}
\begin{equation*}
  \Pr\left[\hat{\mu}_i(t) - \theta_i(t) > \sqrt{4\log t\over N_i(t)}\right] \le {1\over t^2}.
\end{equation*}
\end{lemma}

\begin{proof}
We only need to prove the first inequality, the second one can be done by symmetry.
\begin{eqnarray}
\nonumber&&\Pr\left[\theta_i(t)-\hat{\mu}_i(t)  > \sqrt{4\log t\over N_i(t)}\right]\\
\nonumber =&& \sum_{a,b > 0} \Pr\left[\theta_i(t)-\hat{\mu}_i(t)  > \sqrt{4\log t\over N_i(t)}, a_i(t) = a, b_i(t) = b \right]\\
\nonumber=&& \sum_{a,b > 0} \Pr[a_i(t) = a, b_i(t) = b] \Pr\left[\theta_i(t)-\hat{\mu}_i(t)  > \sqrt{4\log t\over N_i(t)}\mid a_i(t) = a, b_i(t) = b \right]\\
\label{eq_m222}\le&& \sum_{a,b > 0} \Pr[a_i(t) = a, b_i(t) = b] \exp\left(-2{N_i(t)\left(\sqrt{4\log t\over N_i(t)}\right)^2}\right)\\
\nonumber = &&\sum_{a,b > 0} \Pr[a_i(t) = a, b_i(t) = b]\exp(-4\log t)\\
\nonumber=&&  \exp(-4\log t)\\
\nonumber\le&&  {1\over t^2},
\end{eqnarray}
where Eq. \eqref{eq_m222} is given by Lemma \ref{Fact_M15}.
\end{proof}

Now we start the main proof of Theorem \ref{Proposition1}, here
we analyze the regret of Algorithm \ref{Algorithm_TS} on the following MAB instance:

\begin{restatable}{problem}{Problemtwo}\label{Problem2}
	There are totally $m$ arms, $\bm{\mu} = [1,0.5,\cdots,0.5,0.35]$, approximation ratio $\lambda = 0.49$. Given a parameter set $\bm{\theta}$, we say arm $i$ is feasible if $\theta_i \ge \lambda \max_j \theta_j$.
The {\sf Oracle} works as follows: 
(a) if arm $m$ is feasible, and there exist feasible arms in $[2, m-1]$, then {\sf Oracle} returns arm $m$ with probability ${1\over 2}$, and returns one feasible arm in $[2,m-1]$ uniformly with total probability ${1\over 2}$; 
(b) if arm $m$ is feasible and there is no feasible arm in $[2,m-1]$, then {\sf Oracle} returns arm $m$; 
(c) if arm $m$ is not feasible but there exist feasible arms in $[2,m-1]$, then {\sf Oracle} returns one feasible arm in $[2,m-1]$ uniformly;
(d) when all arms in $[2,m]$ are not feasible, {\sf Oracle} returns arm $1$.

\end{restatable}

In this example, the approximation regret of pulling arm $i \in [2,m-1]$ is $0.49 \times 1 - 0.5 = -0.01$, and the approximation regret of pulling arm $m$ is $0.49 \times 1 - 0.35 = 0.14$.

The key idea is that we may never play the true best arm (arm $1$ above) 
	when using the approximation oracle. 
When the sample of the best arm from the prior distribution is good (i.e., close to 1), we choose an approximate arm (arm $2,3,\cdots,m-1$ above) but not the best arm; 
	otherwise we choose the bad arm $m$ with a positive approximation regret.
In this case, the expected approximation regret of each time slot
depends on whether the prior distribution of the best arm at the beginning of this time slot
is good or not.
Since the best arm is never observed, we never update its prior distribution. Thus the expected regret in each time slot can remain a positive constant forever. 

In our conference paper \cite{WC18}, we only prove that arm $m$ in Problem \ref{Problem2} will be pulled for $\Theta(T)$ number of times. However, this is not enough to achieve an $\Omega(T)$ regret lower bound, since pulling an arm $i \in [1,m-1]$ leads to a negative regret and this may compensate for the $\Theta(T)$ positive regret of pulling arm $m$. In this manuscript, we correct this mistake by using a more meticulous analysis, which is explained in detail as follows.

We first define the following event for any $2 \le i \le m-1$:
\begin{equation*}
\mathcal{J}_i(n) = \{\forall t=1,2,\ldots, (N_i(t) \ge n) \rightarrow (\hat{\mu}_i(t) \ge \mu_i - 0.005) \}.
\end{equation*}

By Chernoff-Hoeffding inequality (Fact \ref{Fact_Chernoff}), we know that for any $2 \le i \le m-1$,
\begin{eqnarray}
\nonumber \Pr[\mathcal{J}_i(n)] &\ge& 1 - \sum_{t=1}^\infty\sum_{n' \ge n}  \Pr[N_i(t) = n', \hat{\mu}_i(t) \ge \mu_i - 0.005]\\
\nonumber&=& 1 - \sum_{n' \ge n} \sum_{t=1}^\infty \Pr[N_i(t) = n', \hat{\mu}_i(t) \ge \mu_i - 0.005]\\
\nonumber &\ge &1 - \sum_{n'=n}^\infty \exp(-2n'\cdot(0.005)^2) \\
\nonumber &=& 1-\sum_{n'=n}^\infty \exp(-n'/20000) \\
\label{Eq_901}&\ge& 1-20000\exp(-n/20000).
\end{eqnarray}

Then we can choose $n_1$ such that $1-20000\exp(-n_1/20000) \ge 0.9$, this means that $\Pr[\mathcal{J}_i(n_1)] \ge 0.9$. We emphasize that $n_1$ is a constant that does not depend on $m$. 

Let $A_1(n)$ denote the set of arms $2\le i \le m-1$ such that $\mathcal{J}_i(n)$ holds. Then we define the next event $\mathcal{K}_1(n) $ as
\begin{equation*}
\mathcal{K}_1(n) = \left\{|A_1(n)| \ge {m-2\over 2} \right\}.
\end{equation*}

Since the events $\mathcal{J}_i(n)$'s are independent for different $i$'s and recall that $\Pr[\mathcal{J}_i(n_1)] \ge 0.9$,  by Chernoff-Hoeffding inequality (Fact \ref{Fact_Chernoff}), we have that
\begin{equation}\label{Eq_979}
\Pr[\mathcal{K}_1(n_1)] \ge 1 - \exp(-2(m-2)\times (0.9 - 0.5)^2) = 1 - \exp(-0.32(m-2)).
\end{equation}

Recall that $F^{Beta}_{a,b}(x)$ is the CDF of beta distribution with parameter $(a,b)$.
For any action $i \in A_1(n_1)$ and any $t$ with $N_i(t) < n_1$, 
	the probability that $\theta_i(t) > 0.49$ is $1 - F^{Beta}_{a_i(t),b_i(t)}(0.49)
	\ge  1 - F^{Beta}_{1,n_1}(0.49) = 0.51^{n_1}$.
Let $p(n_1) \triangleq 0.51^{n_1}$ be this probability lower bound.
When $\theta_i(t) > 0.49$, we know arm $i$ must be feasible for round $t$, which means it is chosen by the {\sf Oracle} 
	with probability at least ${1\over 2(m-2)}$ in round $t$, regardless of other randomness in the system. 

Thus, when considering the random variables $N_i(0), N_i(1),\cdots,N_i(t)$ for $i \in [2,m-1]$, we have that $\E[N_i(t)|N_i(0),\cdots N_i(t-1)] \ge N_i(t-1) + {p(n_1) \over 2(m-2)}$ if $N_i(t-1) < n_1$. To deal with the cases that $N_i(t-1) \ge n_1$, we define another series of random variables $N_i'(0), N_i'(1), \cdots , N_i'(t)$ as follows: if $N_i(\tau-1) < n_1$, we set that $N_i'(\tau) = N_i(\tau)$; otherwise we set that $N_i'(\tau) = N_i'(\tau-1) + 1$. Under this definition, we could see that $\E[N_i'(t)|N_i'(0),\cdots,N_i'(t-1)] \ge N_i'(t-1) + {p(n_1) \over 2(m-2)}$ for any $N_i'(t-1) < n_1$, and $\E[N_i'(t)|N_i'(0),\cdots,N_i(t-1)] = N_i'(t-1) + 1 \ge N_i'(t-1) + {p(n_1) \over 2(m-2)}$ for any $N_i'(t-1) \ge  n_1$. 
Thus, we could use the Azuma's inequality (Fact \ref{Fact_Azuma}), and get that
\begin{equation} \label{eq:Ntconcentration}
\Pr\left[N_i'(t) \le {t\cdot p(n_1) \over 2(m-2)} - \epsilon\right] \le \exp\left(-{\epsilon^2 \over 2t} \right).
\end{equation}

Under our definition, $\{N_i'(t) \ge n_1\} \iff \{N_i(t) \ge n_1\}$, this implies that for any $t$ with ${t\cdot p(n_1)\over 4(m-2)} \ge n_1$, we can obtain 
\begin{eqnarray}
\nonumber\Pr[N_i(t) \le n_1] &=& \Pr[N_i'(t) \le n_1] \\
\nonumber&\le& \Pr\left[N_i'(t) \le {t\cdot p(n_1) \over 4(m-2)}\right]\\
\nonumber&=& \Pr\left[N_i'(t) \le {t\cdot p(n_1) \over 2(m-2)} - {t \cdot p(n_1) \over 4(m-2)}\right]\\
\label{Eq_7777}&\le& \exp\left(-{1\over 2t} \cdot \left({t \cdot p(n_1) \over 4(m-2)}\right)^2 \right)\\
\nonumber&=& \exp\left(-  {t \cdot (p(n_1))^2 \over 32(m-2)^2} \right),
\end{eqnarray}
where Eq. \eqref{Eq_7777} is given by applying inequality~\eqref{eq:Ntconcentration}.

Now we define event $\mathcal{K}_2(n_2)$ as:
\begin{equation*}
\mathcal{K}_2(n_2) = \{\forall i \in A_1(n_1), (t\ge n_2)\rightarrow ( N_i(t) \ge n_1)\}.
\end{equation*}

Then we know that for any $n_2$ such that ${n_2\cdot p(n_1)\over 4(m-2)} \ge n_1$, 
\begin{eqnarray}
\nonumber \Pr[\mathcal{K}_2(n_2)] &\ge& 1 - \sum_{i\in A_1(n_1)}\sum_{t=n_2}^\infty \exp\left(-{t\cdot (p(n_1))2 \over 32(m-2)^2}\right)\\
\nonumber &\ge& 1 - (m-2)\sum_{t=n_2}^\infty \exp\left(-{t\cdot (p(n_1))2 \over 32(m-2)^2}\right)\\
\label{Eq_978}&\ge& 1 - {32(m-2)^3 \over (p(n_1))^2} \exp\left(-{n_2\cdot (p(n_1))^2 \over 32(m-2)^2}\right).
\end{eqnarray}

In the next step, we define the following three events:
\begin{itemize}
\item $\mathcal{L}_1(n_3) = \{\forall 2\le i \le m-1, (t \ge n_3) \rightarrow (|\hat{\mu}_i(t) - \mu_i|\le \sqrt{4\log t \over N_i(t)},  |\theta_i(t) - \hat{\mu}_i(t)|\le \sqrt{4\log t \over N_i(t)})\}$;
\item $\mathcal{L}_2(n_3) = \{\forall i \in A_1(n_1), ( t \ge n_3 ) \rightarrow (\sqrt{4\log t \over N_i(t)} \le 0.005)\}$;
\item $\mathcal{L}_3(n_3) = \{N_1(n_3)  = 0 \}$.
\end{itemize}

By Chernoff-Hoeffding inequality (Fact \ref{Fact_Chernoff}), we have that for any $2 \le i \le m-1$,
\begin{equation*}
\Pr\left[|\hat{\mu}_i(t) - \mu_i|\le \sqrt{4\log t \over N_i(t)}\right] \le {2\over t^2}.
\end{equation*}

By Lemma \ref{Fact_M5}, we also have that for any $2 \le i \le m-1$,
\begin{equation*}
\Pr\left[|\theta_i(t) - \hat{\mu}_i(t)|\le \sqrt{4\log t \over N_i(t)}\right] \le {2\over t^2}.
\end{equation*}

Thus
\begin{equation}\label{Eq_977}
\Pr[\mathcal{L}_1(n_3)] \ge 1 - \sum_{t=n_3}^\infty {4(m-2)\over t^2} \ge 1 - {4(m-2)\over n_3}.
\end{equation}

Now we consider the probability $\Pr[\mathcal{L}_2(n_3)|\mathcal{K}_2(n_2)]$ with $n_2 < n_3$. When $\mathcal{K}_2(n_2)$ happens, for any arm $i \in A_1(n_1)$ and $t > n_2$, we have that (here $P(a,b)$ denotes $\Pr[a_i(t)=a, b_i(t)=b\mid\mathcal{K}_2(n_2)]$)
\begin{eqnarray}
\nonumber&&\Pr[\theta_i(t) \ge 0.49 \mid \mathcal{K}_2(n_2)]\\
\nonumber  =&& \sum_{a,b>0} \Pr[\theta_i(t) \ge 0.49, a_i(t)=a, b_i(t)=b\mid \mathcal{K}_2(n_2)]\\
\nonumber=&&  \sum_{a,b>0}P(a,b)\Pr[\theta_i(t) \ge 0.49\mid \mathcal{K}_2(n_2),a_i(t)=a, b_i(t)=b]\\
\label{Eq_921}\ge&&  \sum_{a,b>0}P(a,b)\Pr[\theta_i(t)  \ge \hat{\mu}_i(t) - 0.005 \mid \mathcal{K}_2(n_2),a_i(t)=a, b_i(t)=b]\\
\label{Eq_922}\ge&&  \sum_{a,b>0}P(a,b)(1 - \exp(-2N_i(t) \cdot 0.005^2))\\
\label{Eq_923}\ge&& \sum_{a,b>0}P(a,b)(1 - \exp(-2n_1 \cdot 0.005^2))\\
\nonumber\ge&& 1 - \exp(-2n_1 \cdot 0.005^2)\\
\label{Eq_924}\ge&& 0.9\\
\nonumber \ge&& 0.5.
\end{eqnarray}
Eq. \eqref{Eq_921} is because that under $\mathcal{K}_2(n_2)$, $\forall i \in A_1(n_1), (t\ge n_2)\rightarrow (N_i(t) \ge n_1)$, and by definition of $A_1(n_1)$, we know that $\hat{\mu}_i(t) \ge \mu_i - 0.005 \ge 0.495$, this implies that $(\theta_i(t) \ge 0.49) \rightarrow (\theta_i(t)  \ge \hat{\mu}_i(t) - 0.005)$.
Eq. \eqref{Eq_922} is given by Lemma \ref{Fact_M15}, since $\theta_i(t)$ only depends on $a_i(t), b_i(t)$.
Eq. \eqref{Eq_923} is because that under $\mathcal{K}_2(n_2)$, $N_i(t) \ge n_1$ and Eq. \eqref{Eq_924} is because that by definition of $n_1$, we have $1 - \exp(-2n_1 \cdot 0.005^2) \ge 0.9$. 

An arm $i \in [2,m-1]$ with $\theta_i(t)\ge 0.49$ must be feasible. Thus, arm $i \in [2,m-1]$ is pulled with probability at least $0.5 \cdot {1\over 2(m-2)} = {1\over 4(m-2)}$ (under event $\mathcal{K}_2(n_2)$). Now consider the random variables $N_i''(0) = N_i(n_2) - N_i(n_2), N_i''(1) = N_i(n_2+1) - N_i(n_2), \cdots, N_i''(t) = N_i(n_2+t) - N_i(n_2)$, we can see that $\E[N_i''(t)| N_i''(0), \cdots, N_i''(t-1),\mathcal{K}_2(n_2)] \ge N_i''(t-1) + {1\over 4(m-2)}$. 

Because of this, we can use the Azuma's inequality (Fact \ref{Fact_Azuma}) again and get that
\begin{equation*}
\Pr\left[N_i''(t) \le {t - n_2\over 4(m-2)} - \epsilon\mid\mathcal{K}_2(n_2) \right] \le \exp\left(-{\epsilon^2 \over 2(t - n_2)}\right).
\end{equation*}

Notice that $\neg \mathcal{L}_2(n_3)$ implies that there exists $i \in A_1(n_1)$ and $t > n_3$ such that $N_i(t) \le 160000 \log t$. Thus, for any $n_3$ such that ${ n_3-n_2\over 8(m-2)} \ge 160000 \log n_3$, we have
\begin{eqnarray}
\nonumber \Pr[\mathcal{L}_2(n_3)|\mathcal{K}_2(n_2)] &\ge& 1 - \sum_{i \in A_1(n_1)} \sum_{t=n_3}^\infty \Pr[N_i(t) \le 160000 \log t\mid\mathcal{K}_2(n_2)]\\
\nonumber&\ge& 1 - \sum_{i \in A_1(n_1)} \sum_{t=n_3}^\infty \Pr\left[N_i''(t) \le 160000 \log t \mid\mathcal{K}_2(n_2)\right]\\
\nonumber&\ge& 1 - \sum_{i \in A_1(n_1)} \sum_{t=n_3}^\infty \Pr\left[N_i''(t) \le {(t-n_2)\over 8(m-2)} \mid\mathcal{K}_2(n_2)\right]\\
\nonumber&\ge& 1 - \sum_{i \in A_1(n_1)} \sum_{t=n_3}^\infty \Pr\left[N_i''(t) \le {(t-n_2)\over 4(m-2)} -  {(t-n_2)\over 8(m-2)}\mid\mathcal{K}_2(n_2)\right]\\
\nonumber&\ge& 1 - \sum_{i \in A_1(n_1)} \sum_{t=n_3}^\infty \exp\left(-{(t-n_2) \over 128(m-2)^2}\right)\\
\nonumber &\ge& 1 - (m-2) \sum_{t=n_3}^\infty \exp\left(-{(t-n_2) \over 128(m-2)^2}\right)\\
\label{Eq_976}&\ge& 1 - 128(m-2)^3 \exp\left(-{(n_3-n_2) \over 128(m-2)^2}\right).
\end{eqnarray}

Now we consider $\Pr[\mathcal{L}_3(n_3)|\mathcal{K}_1(n_1) \land \mathcal{K}_2(n_2)]$ with $n_2 < n_3$. Notice that we will pull arm $1$ only if all the other arms are not feasible, and under event $\mathcal{K}_2(n_2)$, all the arms $i\in A_1(n_1)$ satisfy that $\Pr[\theta_i(t) \ge 0.49] \ge p(n_1)$ when $N_i(t) < n_1$, and $\Pr[\theta_i(t) \ge 0.49] \ge 0.5$ when $N_i(t) \ge n_1$. Moreover, there are at least ${m-2\over 2}$ arms in $A_1(n_1)$ under event $\mathcal{K}_1(n_1)$. Thus, denote $i(t)$  the chosen arm at time $t$, then for $t < n_2$, we have that
\begin{equation*}
\Pr[i(t) = 1] \le (1-p(n_1))^{|A_1(n_1)|} \le (1-p(n_1))^{(m-2)/2},
\end{equation*}
and for $t \ge n_2$, we have that
\begin{equation*}
\Pr[i(t) = 1] \le 0.5^{|A_1(n_1)|} \le 0.5^{(m-2)/2}.
\end{equation*}

Then we know that
\begin{equation*}
\E[N_1(n_3)|\mathcal{K}_1(n_1) \land \mathcal{K}_2(n_2)] \le n_2(1-p(n_1))^{(m-2)/2} + {n_3 - n_2 \over 2^{(m-2)/2}}.
\end{equation*}

By Markov inequality (Fact \ref{Fact_Markov}), 
\begin{eqnarray}
\nonumber\Pr[\mathcal{L}_3(n_3)|\mathcal{K}_1(n_1) \land \mathcal{K}_2(n_2)] &=&1 - \Pr[N_1(n_3) \ge 1|\mathcal{K}_1(n_1) \land \mathcal{K}_2(n_2)]\\
\nonumber &\ge& 1 - {\E[N_1(n_3) |\mathcal{K}_1(n_1) \land \mathcal{K}_2(n_2)] \over 1}\\
\label{Eq_975} &\ge& 1 - n_2(1-p(n_1))^{(m-2)/2} - {n_3 - n_2 \over 2^{(m-2)/2}}.
\end{eqnarray}

The following lemma shows that the probability that all the five events $\mathcal{K}_1(n_1), \mathcal{K}_2(n_2)$, $\mathcal{L}_1(n_3), \mathcal{L}_2(n_3), \mathcal{L}_3(n_3)$ happen can be arbitrarily close to $1$.

\begin{lemma}\label{Lemma_777}
For any $\delta > 0$, these exists an $(m,n_2, n_3)$ such that
\begin{equation*}
\Pr[\mathcal{K}_1(n_1) \land \mathcal{K}_2(n_2) \land \mathcal{L}_1(n_3) \land \mathcal{L}_2(n_3) \land \mathcal{L}_3(n_3)] \ge 1 - \delta.
\end{equation*}
\end{lemma}

\begin{proof} Notice that
\begin{eqnarray*}
&&\Pr[\mathcal{K}_1(n_1) \land \mathcal{K}_2(n_2) \land \mathcal{L}_1(n_3) \land \mathcal{L}_2(n_3) \land \mathcal{L}_3(n_3)]\\
&\ge& 1 - \Pr[\neg \mathcal{K}_1(n_1)] - \Pr[\neg \mathcal{K}_2(n_2)] - \Pr[\neg \mathcal{L}_1(n_3)] \\
&&- \Pr[\neg \mathcal{L}_2(n_3) \land \mathcal{K}_2(n_2)] - \Pr[\neg \mathcal{L}_3(n_3)\land \mathcal{K}_1(n_1) \land \mathcal{K}_2(n_2)] \\
&\ge&1 - \Pr[\neg \mathcal{K}_1(n_1)] - \Pr[\neg \mathcal{K}_2(n_2)] - \Pr[\neg \mathcal{L}_1(n_3)] \\
&&- \Pr[\neg \mathcal{L}_2(n_3) | \mathcal{K}_2(n_2)] - \Pr[\neg \mathcal{L}_3(n_3)|\mathcal{K}_1(n_1) \land \mathcal{K}_2(n_2)].
\end{eqnarray*}

Based on Eq. \eqref{Eq_979}, \eqref{Eq_978}, \eqref{Eq_977}, \eqref{Eq_976}, and \eqref{Eq_975}, we only need to prove that for any $\delta > 0$, there exists $(m,n_2, n_3)$ pair such that:
\begin{itemize}
\item $\exp(-0.32(m-2)) \le {\delta \over 5}$;
\item ${32(m-2)^3 \over (p(n_1))^2} \exp\left(-{n_2\cdot (p(n_1))^2 \over 32(m-2)^2}\right) \le {\delta \over 5}$ and ${n_2\cdot p(n_1)\over 4(m-2)} \ge n_1$;
\item ${4(m-2)\over n_3} \le {\delta \over 5}$;
\item $128(m-2)^3 \exp\left(-{(n_3-n_2) \over 128(m-2)^2}\right)\le {\delta \over 5}$ and ${ (n_3-n_2)\over 8(m-2)} \ge 160000 \log n_3$;
\item $n_2(1-p(n_1))^{(m-2)/2} +  {n_3 - n_2 \over 2^{(m-2)/2}} \le {\delta \over 5}$.
\end{itemize}

The first one means that we need $m \ge {25 \over 8} \log{5\over \delta} + 2$. 

The second one means that we need $n_2 \ge {32(m-2)^2\over (p(n_1))^2} \log{160(m-2)^3 \over (p(n_1))^2\delta}$ and $n_2 \ge {4(m-2)n_1 \over p(n_1)}$.

The third one means that we need $n_3 \ge {20(m-2)\over \delta} $.

The fourth one means that we need $n_3 \ge 128(m-2)^2 \log{640(m-2)^3 \over \delta} + n_2$, $n_3\ge 2n_2$ and $n_3 \ge 5120000(m-2)\log\left(5120000(m-2)\right)$.

The last one means that we need $n_2 \le {\delta \over 10 (1-p(n_1))^{(m-2)/2}}$ and $n_3 \le n_2 + {2^{(m-2)/2} \delta \over 10}$.

Notice that the only constraint on $m$ is $m \ge {25 \over 8} \log{5\over \delta} + 2$, which is always true when $m \to \infty$. For the constraints on $n_2$ or $n_3$, the upper bounds of them (the last one) grow exponentially as $m$ grows up (recall that $n_1,p(n_1)$ are constants), and the lower bounds (the second, third and fourth ones) grow polynomially as $m$ grows up. Thus, for any $\delta > 0$, there must exist a large enough $m$ and corresponding $n_2,n_3$ such that all these inequalities holds, which means $\Pr[\mathcal{K}_1(n_1) \land \mathcal{K}_2(n_2) \land \mathcal{L}_1(n_3) \land \mathcal{L}_2(n_3) \land \mathcal{L}_3(n_3)] \ge 1 - \delta$.
\end{proof}

Now we consider the approximation regret when all the five events $\mathcal{K}_1(n_1), \mathcal{K}_2(n_2)$, $\mathcal{L}_1(n_3), \mathcal{L}_2(n_3), \mathcal{L}_3(n_3)$ happen. Note that $\mathcal{L}_1(n_3)\land \mathcal{L}_2(n_3)$ means that for any $t > n_3$, $i \in A_1(n_1)$, we have that $\theta_i(t) \in [0.49, 0.51]$. 
Thus arms $i \in A_1(n_1)$ are always feasible after time step $n_3$, which means that we will never choose arm $1$ for $t > n_3$. Moreover, $\mathcal{L}_3(n_3)$ implies that $N_1(n_3) = 0$, which means that we can never pull arm $1$, and its prior distribution remains to be the uniform distribution on $[0,1]$.

When $N_m(t) < \lceil 200 \log 4 \rceil$, we know that there is always a constant probability (even though it is small) that $\theta_m(t) > 0.5$. Since $\theta_m(t) > 0.5$ means that $m$ is feasible, we know that arm $i$ is pulled with a constant probability. Because of this, after a finite number of time steps (at time step $n_4 > n_3$), we can expect that $N_m(t) = \lceil 200 \log 4 \rceil$. After this time slot, we know that 
\begin{eqnarray}
\nonumber&&\Pr[\theta_m(t) \ge 0.25]\\
 \label{Eq_1999}&\ge& 1 - \Pr[\mu_m \ge \hat{\mu}_m(t) + 0.05] - \Pr[\hat{\mu}_m(t) \ge \theta_m(t) + 0.05]\\
\label{Eq_1998}&\ge&1 - \exp(-0.05^2 \cdot 2\lceil 200 \log 4 \rceil )- \exp(-0.05^2 \cdot 2\lceil 200 \log 4 \rceil)\\
\nonumber&\ge& {1\over 2}.
\end{eqnarray}
Eq. \eqref{Eq_1999} is because that $\theta_m(t) < 0.25$ means either $\mu_m \ge \hat{\mu}_m(t) + 0.05$ or $\hat{\mu}_m(t) \ge \theta_m(t) + 0.05$ (recall that $\mu_m = 0.35$), 
and Eq. \eqref{Eq_1998} comes from Chernoff-Hoeffding inequality (Fact \ref{Fact_Chernoff}) and Lemma \ref{Fact_M15}, whose proofs are similar as before.

Since $\theta_1(t)$ follows the uniform distribution on $[0,1]$, and in each time slot all the samples are independent, we have 
\begin{equation*}
\Pr[\theta_m(t) > 0.25, \theta_1(t) < 0.5] \ge {1\over 4}.
\end{equation*}

Then we know that for any $\epsilon > 0$, there exists a $T^{(0)}$ such that for any $T > T^{(0)}$, with probability at least $1-\epsilon$, there are $0.99 \cdot {(T-n_4)\over 4}$ number of times that $\theta_m(t) > 0.25$ and $\theta_1(t) < 0.5$ within period $(n_4,T]$.

Now we define event $\mathcal{U}(t)$ as follows
\begin{equation*}
\mathcal{U}(t) = \{ \theta_m(t) > 0.25, \theta_1(t) < 0.5, \theta_m(t) < 0.49 \max_i \theta_i(t)\},
\end{equation*}
i.e., although $\theta_m(t) > 0.25, \theta_1(t) < 0.5$, arm $m$ is not feasible at time step $t$.

Define $t_1(T) \triangleq \sum_{t=n_4+1}^T \I[\mathcal{U}(t)]$
, then we have the following lemma.

\begin{lemma}\label{Lemma_788}
For any $\epsilon > 0$, there exists an $T^{(1)}$ such that for any $T > T^{(1)}$, with probability at least $1 - \epsilon$, $t_1(T) < 0.99 \cdot {(T-n_4)\over 20}$.
\end{lemma}

\begin{proof} 
%
Let $A_2(t)$ denote the set of arms $i \in [2,m-1]$ with $\sqrt{4\log t \over N_i(t)} \ge 0.005$ at time $t$, and $\mathcal{V}(t)$ denote the event that we pull an arm $i \in A_2(t)$, i.e.,
\begin{equation*}
\mathcal{V}(t) = \{ i(t) \in A_2(t)\},
\end{equation*}
where $i(t)$ is the pulled arm at time slot $t$. Define $t_2(T) \triangleq \sum_{t=n_4+1}^T \I[\mathcal{V}(t)]$ as the number of time steps that we pull an arm $i$ in set $A_2(t)$ during time period $(n_4, T]$.

Notice that when an arm $i$ is pulled for at least $160000 \log T$ number of times, then $\sqrt{4\log t \over N_i(t)} \le \sqrt{4\log T \over 160000 \log T} = 0.005$, which implies that $i \notin A_2(T)$. Because of this, $t_2(T) \le 160000(m-2) \log T$.

Now we consider the probability $\Pr[\mathcal{V}(t)  | \mathcal{U}(t) ]$. Under event $\mathcal{L}_1(n_3)$, all the actions $i \in [2,m-1] \setminus A_2(t)$  satisfy that $\theta_i(t) \le 0.51$. Since $0.51 \times 0.49 \le 0.25$, they are not able to prevent arm $m$ from being feasible. 
Similarly, $\theta_1(t) < 0.5$ means that arm $1$ is not able to prevent arm $m$ from being  feasible too.
Because of this, there must exist $j \in A_2(t)$ such that $\theta_j(t) \ge 0.51$, which means arm $j$ must be feasible. Along with the fact that $m$ is not feasible, we have probability of at least ${1\over m-2}$ to pull an arm $j \in A_2(t)$, i.e.,
\begin{equation*}
\Pr[\mathcal{V}(t)  | \mathcal{U}(t)] \ge {1\over m-2}.
\end{equation*}

Thus
\begin{eqnarray}
\nonumber\E[t_2(T)] &=& \sum_{t=n_4+1}^{T} \Pr[\mathcal{V}(t)]\\
\nonumber&=&\sum_{t=n_4+1}^{T} \Pr[\mathcal{V}(t) \land  \mathcal{U}(t)] + \sum_{t=n_4+1}^{T} \nonumber\Pr[\mathcal{V}(t) \land \neg \mathcal{U}(t)]\\
\nonumber&\ge &\sum_{t=n_4+1}^{T} \Pr[\mathcal{V}(t) \land  \mathcal{U}(t)]\\
\nonumber&=& \sum_{t=n_4+1}^{T} \Pr[\mathcal{V}(t) | \mathcal{U}(t)] \Pr[\mathcal{U}(t)]\\
\label{Eq_6666}&\ge&\sum_{t=n_4+1}^{T}{1\over m-2}\Pr[\mathcal{U}(t)]\\
\nonumber&=&{1\over m-2}\sum_{t=n_4+1}^{T}\Pr[\mathcal{U}(t)]\\
\nonumber&=&{1\over m-2}\E[t_1(T)],
\end{eqnarray}
where Eq. \eqref{Eq_6666} is because that $\Pr[\mathcal{V}(t)  | \mathcal{U}(t)] \ge {1\over m-2}$.

Since $t_2(T) \le 160000(m-2) \log T$, we have $\E[t_2(T)] \le 160000(m-2) \log T$, which implies that
\begin{equation*}
\E[t_1(T)] \le (m-2)\E[t_2(T)]\le 160000(m-2)^2 \log T.
\end{equation*}

By Markov inequality (Fact \ref{Fact_Markov}), we have that
\begin{equation*}
\Pr\left[t_1(T) \ge 0.99 \cdot {(T-n_4) \over 20}\right] \le {\E[t_1(T)] \over 0.99 \cdot {(T-n_4) \over 20}} \le {3200000(m-2)^2\log T \over 0.99 (T-n_4)}.
\end{equation*}

Thus, for any $\epsilon > 0$, there must exist a $T^{(1)}$ such that for any $T > T^{(1)}$, $\Pr[t_1(T) < 0.99 \cdot {(T-n_4)\over 20}] \ge 1 - \epsilon$.
\end{proof}

By Lemma \ref{Lemma_788}, we know that for any $T > \max\{T^{(0)}, T^{(1)}\}$, with probability at least $1 - 2\epsilon$, there are $0.99 \cdot {(T-n_4)\over 4} - 0.99 \cdot {(T-n_4)\over 20} = 0.99 \cdot {(T-n_4)\over 5}$ number of time steps such that arm $m$ is feasible (i.e., $\theta_m(t) \ge 0.49 \max_i \theta_i(t)$) during period $(n_4, T]$. Note that arm $m$ is pulled with probability ${1\over 2}$ when it is feasible. Then by Chernoff-Hoeffding inequality (Fact \ref{Fact_Chernoff}), we know that for any $\epsilon > 0$, there must exist $T^{(2)} > \max\{T^{(0)}, T^{(1)}\} $ such that for any $T > T^{(2)}$, with probability at least $1-3\epsilon$, there are $0.98 \cdot {(T-n_4)\over 10}$ of pulls on arm $m$ during period $(n_4, T]$.

Because of this, the expected number of times that we pull arm $m$ during period $(n_4, T]$ is at least $(1-3\epsilon) \cdot 0.98 \cdot {(T-n_4)\over 10}$ for $T > T^{(2)}$. Let $T^{(3)} = 11n_4$, and choosing $\epsilon = 0.01$. Then for $T > \max\{T^{(2)}, T^{(3)}\}$, $(1-3\epsilon) \cdot 0.98 \cdot {(T-n_4)\over 10} \ge {0.95T\over 11}$.

Thus, when all the five events $\mathcal{K}_1(n_1), \mathcal{K}_2(n_2)$, $\mathcal{L}_1(n_3), \mathcal{L}_2(n_3), \mathcal{L}_3(n_3)$ happen, there are at least $ {0.95T\over 11}$ pulls on arm $m$ in expectation (for time horizon $T > \max\{T^{(2)}, T^{(3)}\}$), while all the other pulls should be in arm set $[2,m-1]$. Recall that the approximation regret of arm $m$ is $0.14$, and the approximation regret of arm $i\in [2,m-1]$ is $-0.01$, we know that the approximation regret until $T$ is lower bounded by 
\begin{equation*}
\left({0.95T\over 11} \times 0.14 + {10.05T\over 11} \times (-0.01)\right) = {325T\over 110000}.
\end{equation*}



Based on Lemma \ref{Lemma_777}, for any $\delta > 0$, there exists a problem instance such that applying Algorithm \ref{Algorithm_TS} on this instance will cause an approximation regret of at least
\begin{equation*}
{325T\over 110000}(1-\delta) + (-T)\delta = {325-110325\delta \over 110000}T.
\end{equation*}

Choosing $\delta = 0.001$, we know that the approximation regret is $\Theta(T)$.

%
%
%
%




\subsection{The Exponential Constant Term} \label{sec:exponentialconstant}

In this subsection, we show that the exponential constant regret term is unavoidable for applying Thompson sampling in CMAB model. 

Since every arm's sample $\theta_i(t)$ is chosen independently, the worst case is that we need all the samples for base arms in the best super arm to be close to their true means to choose that super arm. Under this case, the probability that we have no regret in each time slot is exponentially (small) with $k^*$, thus we will have such a constant term.
\begin{restatable}{theorem}{Theoremtwo}\label{Theorem_TS_EXP}
	There exists a CMAB instance such that the regret of Algorithm \ref{Algorithm_TS}
		on this instance is at least $\Omega(\min\{2^{k^*}, T\})$.
\end{restatable}

\begin{proof}
We analyze the regret of Algorithm \ref{Algorithm_TS} on the following CMAB instance:

\begin{restatable}{problem}{Problem}\label{Problem1}
	$m = k^* + 1$, there are only two super arms in $\mathcal{I}$, where $S_1 = \{1,2,\cdots,k^*\}$ and $S_2 = \{k^*+1\}$. The mean vector $\bm{\mu}_{S_1} = [1,\cdots,1]$. The reward function $R$ follows $R(S_1,\bm{X}) = \prod_{i\in S_1} X_i$, and $R(S_2,\bm{X}) = 1-\Delta$, while $\Delta = 0.5$.
	The distributions $D_i$ are all independent Bernoulli distributions with mean $\mu_i$ (since $\mu_i = 1$, the observations are always 1).
\end{restatable}


We can use $T_1$ to represent the first time step that we choose to pull super arm $S_1$, and the regret is at least $\Delta\min \{\E[T_1] - 1, T\}$. Thus it is sufficient to prove that $\E[T_1] = \Omega(2^{k^*})$.

Firstly, we know from $R(S_1,\bm{X}) = \prod_{i\in S_1} X_i$ that $r(S_1,\bm{\mu}) = \prod_{i\in S_1} \mu_i$.

Notice that $R(S_2,\bm{\theta}(t))$ is always 0.5, so we only need to consider the probability that $r(S_1,\bm{\theta}(t)) \ge 0.5$.

When there are no observations on base arms $1,\cdots,k^*$, the prior distributions for them all have mean $0.5$, thus $\E[r(S_1,\bm{\theta}(t))] = {1\over 2^{k^*}}$. By Markov inequality (Fact \ref{Fact_Markov}), $\Pr[r(S_1,\bm{\theta}(t)) \ge 0.5] \le {1\over 0.5 \times 2^{k^*}}$, which implies
\begin{equation*}
  \E[T_1] \ge {1\over 2} \times 2^{k^*} = \Omega(2^{k^*}).
\end{equation*}
\end{proof}


The exponential term comes from the bad prior distribution at the beginning of the algorithm. In fact, from the proof of Theorem \ref{theorem_1} we know that if we can pull each base arm for $\tilde{O}({1\over \varepsilon^2})$ times at the beginning and run the CTS policy with prior distribution at the beginning to be the posterior distribution after those observations, then we can reduce the exponential constant term to $O({m\over \epsilon^4})$. However, since $\varepsilon$ depends on $\Delta_{\min}$, which is unknown to the player, we can not simply run each base arm for a few time steps to avoid the exponential constant regret term. 
Perhaps an adaptive choice can be used here, and this is a further research item.

\section{Matroid Bandit Case}


In matroid bandit, we suppose the oracle we use is the greedy one, i.e., given a parameter vector $\bm{\theta}$, it starts from an empty set $S$, and then in each step $k$ adds the arm $i^{(k)} = \argmax_{i\in [m]\setminus S, S\cup \{i\} \in \mathcal{I}} \theta_i$ into $S$, finally stops when no base arm can be included.
Existing results show that this greedy policy is an exact offline oracle \cite{edmonds1971matroids}.
	%




\subsection{Regret Upper Bound}
Let $S^* \in \argmax_{S\in \mathcal{I}} r(S,\bm{\mu})$
be one of the optimal super arm.
Define  $\Delta_i = \min_{j \mid j\in S^*, \mu_j > \mu_i} \mu_j - \mu_i$. If $i\notin S^*$ but $\{j \mid j\in S^*, \mu_j > \mu_i\} = \emptyset$, we define $\Delta_i = \infty$, so that ${1\over \Delta_i} = 0$.
Let $K = \max_{S\in {\cal I}} |S| = |S^*|$.

Notice that the reward function is linear in matroid bandit case. Thus Assumption~\ref{Assumption2} holds with $r(S,\bm{\mu}) = \sum_{i\in S} \mu_i$, and Assumption~\ref{Assumption_Continouos1} holds with $B=1$ naturally. Besides, in matroid bandit we do not need the independent assumption (Assumption \ref{Assumption_IND}) as well. This means that there are no further constraints on the following theorem, which states the regret upper bound of CTS algorithm under the matroid bandit case:


\begin{restatable}{theorem}{TheoremFive}\label{Theorem_TS_MB}
The regret upper bound of Algorithm \ref{Algorithm_TS} for a matroid bandit is:
  \begin{equation*}
    Reg(T) \le \sum_{i\notin S^*}{4\log T\over \Delta_i - 2\varepsilon}{\Delta_i - \varepsilon \over \Delta_i - 2\varepsilon} +
	    \alpha_2 \cdot \left({m\over \varepsilon^4} \right)+ m^2
  \end{equation*}
  for any $\varepsilon > 0$ such that $\forall i \notin S^*$ with $\Delta_i > 0$, $\Delta_i - 2\varepsilon > 0$, where $\alpha_2$ is a constant not dependent
  on the problem instance.
\end{restatable}







Notice that we do not need the outcome distributions of all the base arms to be independent due to the special structure of matroid. When $\varepsilon$ is small, the leading $\log T$ term of the above regret bound
	matches the regret lower bound $\sum_{i\notin S^*} {1\over \Delta_i}\log T$ given by \cite{KWAEE14}.
For the constant term, we have an $O({1\over \varepsilon^4})$ factor while \cite{agrawal2013further} have an $O({1\over \varepsilon^2})$ factor
	in their theorem.
However, even following their analysis, we can only obtain $O({1\over \varepsilon^4})$ and cannot recover the $O({1\over \varepsilon^2})$
	in their analysis.

%

We now explain the reason that we can remove the independence assumption in matroid bandit setting. The detailed proof is referred to \ref{Section_A3}, since it is almost the same as our conference paper.


Firstly, we introduce a fact from \cite{KWAEE14}.

\begin{restatable}{fact}{FactKWAEE}(Lemma $1$ in \cite{KWAEE14})\label{Fact_Bijection_zheng}
  For each $S(t) = \{i^{(1)}(t),\cdots,i^{(K)}(t)\}$ chosen by Algorithm \ref{Algorithm_TS} (the superscript is the order when they are chosen by the greedy policy), we could find a
  bijection $L_t$ from $\{1,2,\ldots, K\}$ to $S^*$ such that:

  1) If $i^{(k)}(t) \in S^*$, then $L_t(k) = i^{(k)}(t)$;

  2) $\forall 1 \le k \le K$, $\{i^{(1)}(t),\cdots,i^{(k-1)}(t),L_t(k)\} \in \mathcal{I}$.
\end{restatable}

With a bijection $L_t$, we could decouple the regret of playing one action $S(t)$ to each pair of mapped arms between
	$S(t)$ and $S^*$, i.e., the regret of time $t$ is $\sum_{k=1}^K \mu_{L_t(k)} - \mu_{i^{(k)}(t)}$.

Note that under this decoupling method, if $i^{(k)}(t) = i$ and $L_t(k) = j$, then both arm $i$ and arm $j$ must be feasible when the greedy policy needs to add the $k$-th base arm into $S(t)$, i.e., we will choose arm $i$ as $i^{(k)}(t)$ only when $\theta_i(t) \ge \theta_j(t)$. This is almost the same as the classic MAB setting, where we will choose sub-optimal arm $i$ only when $\theta_i(t) \ge \theta_1(t)$ (assume arm 1 is optimal).  Fact \ref{Fact_Bijection_zheng} enable us to compare two base arms instead of two super arms in one step of the greedy policy. Hence we do not need the independence assumption.


\section{Experiments}\label{Section_e}

We conduct some preliminary experiments to empirically evaluate the performance of CTS versus CUCB, C-KL-UCB and ESCB. The reason that we choose C-KL-UCB is that: a) in classical MAB model, KL-UCB behaves better than UCB; b) similar with TS, it is also a policy based on Bayes' rule. 
We also make simulations on CUCB and C-KL-UCB with chosen parameters, represented by CUCB-m and C-KL-UCB-m. 
In CUCB, we choose the confidence radius to be $\rad_i(t) = \sqrt{3\log t\over 2N_i(t)}$, while in CUCB-m, it is $\sqrt{\log t\over 2N_i(t)}$. 
In C-KL-UCB we choose $f(t) = \log t + 2\log \log t$ (and the KL-upper confidence bound is given by $\argmax_{q} N_i(t)KL(\hat{\mu}_i(t), q) \le f(t)$), while in C-KL-UCB-m it is $\log t$.
 Those chosen parameters in CUCB-m and C-KL-UCB-m make them behave better, but lack performance analysis.

\subsection{Matroid Bandit}

It is well known that spanning trees form a matroid.
Thus, we test the maximum spanning tree problem as an example of matroid bandits, where edges are arms, and
	super arms are forests.

We first generate a random graph with $M$ nodes, and each pair of nodes has an edge with probability $p$.
If the resulting graph has no spanning tree, we regenerate the graph again.
The mean of the distribution $D_i$ is randomly and uniformly chosen from $[0,1]$.
The expected reward for any spanning tree is the sum of the means of all edges in it.
It is easy to see that this setting is an instance of the matroid bandit.

The results are shown in Figure \ref{Figure_tree} with the probability $p = 0.6$ and $M = 30$.
In Figure \ref{fig:subfig:a}, we set all the arms to have independent distributions. In Figure \ref{fig:subfig:b}, each time slot we generate a global random variable $rand$ uniformly in $[0,1]$, all edges with mean larger than $rand$ will have outcome $1$, while others have outcome $0$. In other words, the distributions of base arms are correlated.
We can see that CTS has smaller regret than CUCB, CUCB-m and C-KL-UCB in both two experiments. As for C-KL-UCB-m algorithm, it behaves better with small $T$, but loses when $T$ is very large. We emphasize that C-KL-UCB-m policy uses parameters without theoretical guarantee, thus CTS algorithm is a better choice.



\begin{figure} 
\centering 
\subfigure[]{ \label{fig:subfig:a} 
\includegraphics[width=2.0in]{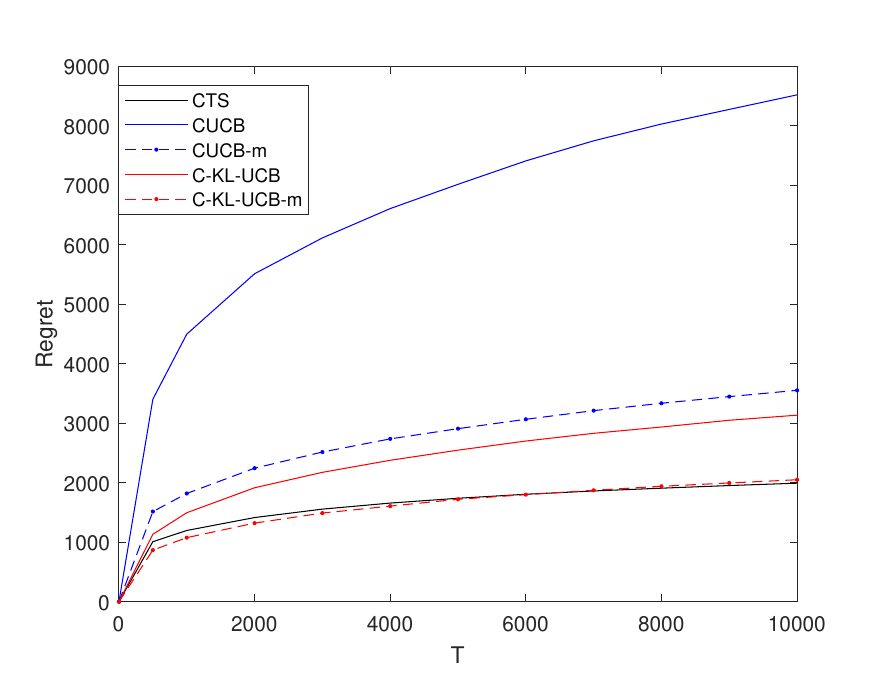}}
 \subfigure[]{ \label{fig:subfig:b} 
\includegraphics[width=2.0in]{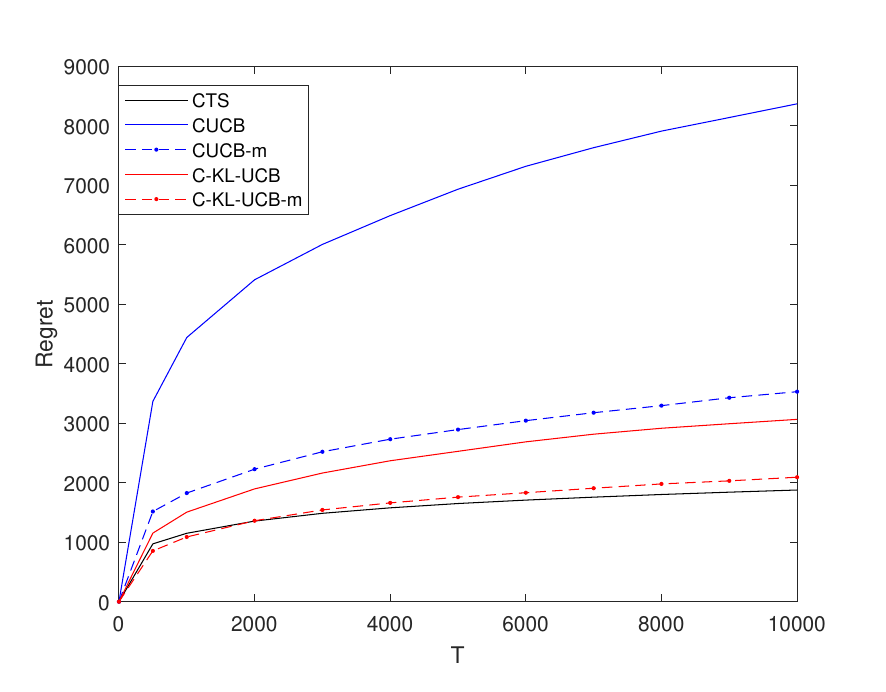}}
 \caption{Experiments on matroid bandit: Maximum Spanning Tree} \label{Figure_tree}  
\end{figure}

\subsection{General CMAB with Linear Reward Function}


In the general CMAB case, we consider two kinds of problem.

\subsubsection{The Shortest Path}

We first consider the shortest path problem. We build two graphs 
 for this experiment, the results of them are shown in Figure \ref{fig:subfig:c} and Figure \ref{fig:subfig:d}. 
The cost of a path is the sum of all edges' mean in that path, while the outcome of each edge $e$ follows an independent Bernoulli distribution with mean $\mu_e$.
The objective is to find the path with minimum cost.
To make the problem more challenging, in both graphs we construct a lot of paths from the source node $s$ to
	the sink node $t$ that only have a little larger cost than the optimal one, and some of them are totally disjoint with the optimal path.


Similar to the case of matroid bandit,
the regret of CTS is also much smaller than that of CUCB, CUCB-m and C-KL-UCB, especially when $T$ is large. As for the C-KL-UCB-m algorithm, although it behaves best in the four UCB-based policies, it still has a large difference between CTS.


\begin{figure} 
\centering 
\subfigure[]{ \label{fig:subfig:c} 
\includegraphics[width=2.0in]{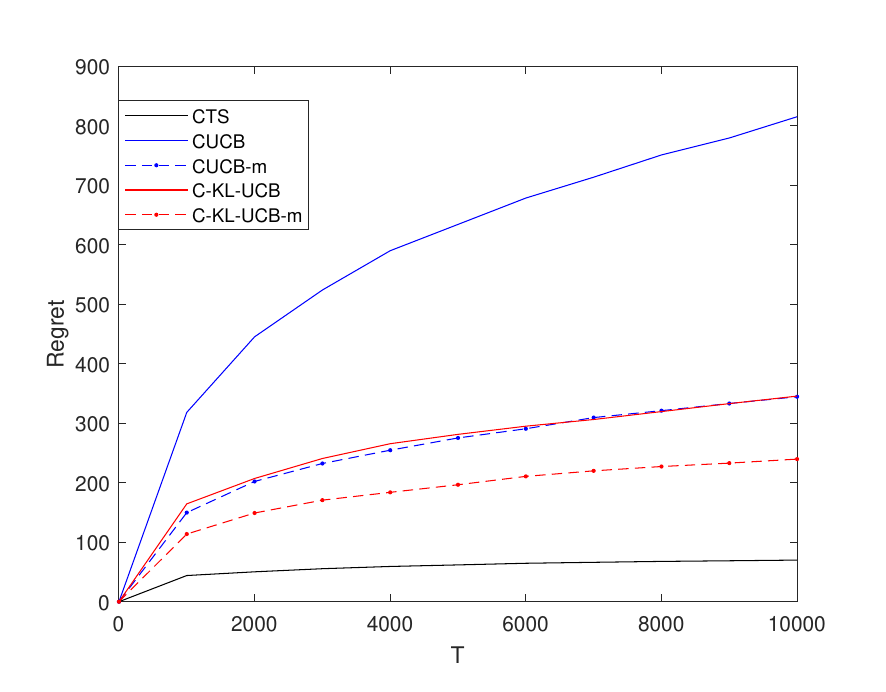}}
 \subfigure[]{ \label{fig:subfig:d} 
\includegraphics[width=2.0in]{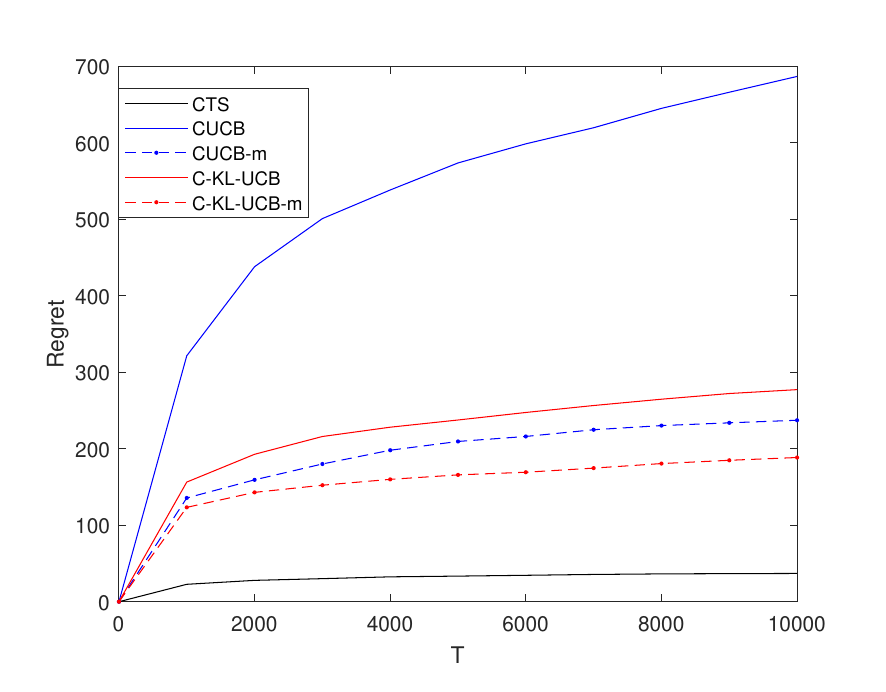}}
 \caption{Experiments on general CMAB: Shortest Path} \label{Figure_path}  
\end{figure}

\subsubsection{Compare with ESCB}

Since ESCB policy needs to compute the upper confidence bounds for every super arm, its time complexity is too large when applying in the shortest path problem. Thus in this section, we do not use such a combinatorial structure. In this experiment, there are totally 20 base arms, and each super arm contains 5 different base arms. 
In Figure \ref{fig:subfig:e}, there are totally 10 available super arms, and in Figure \ref{fig:subfig:f}, there are 100 available super arms. 

We compare the CTS policy with the two kinds of ESCB policies in \cite{Combes2015}, and we can see that CTS behaves better than both the ESCB policies, while all of them are better than CUCB. 

\begin{figure} 
\centering 
\subfigure[]{ \label{fig:subfig:e} 
\includegraphics[width=2.0in]{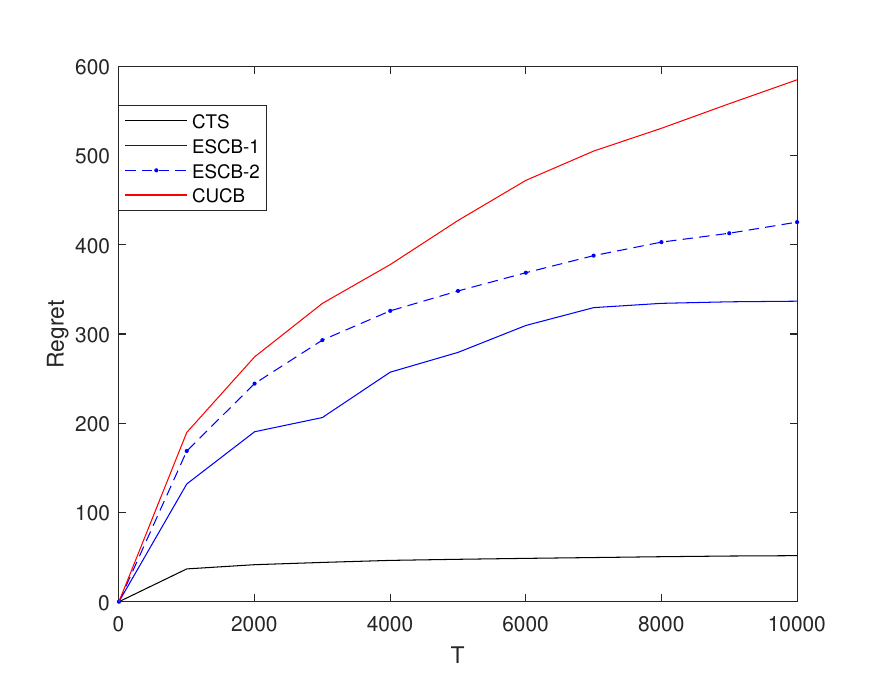}}
 \subfigure[]{ \label{fig:subfig:f} 
\includegraphics[width=2.0in]{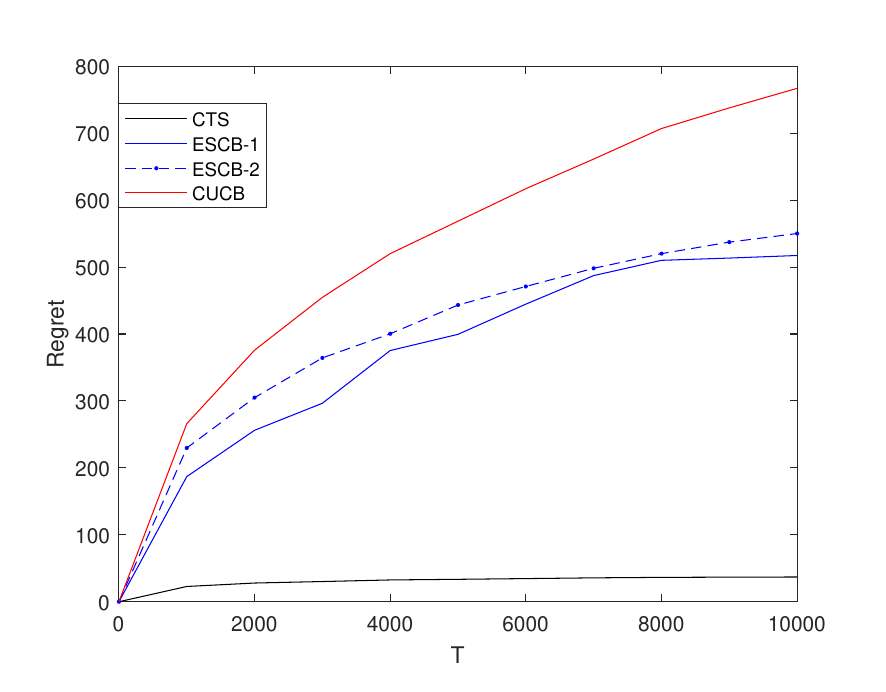}}
 \caption{Experiments on general CMAB: Shortest Path} \label{Figure_path}  
\end{figure}

\section{Conclusion and Future Work}

In this paper, we apply combinatorial Thompson sampling to combinatorial multi-armed bandit and matroid bandit problems,
	and obtain theoretical regret upper bounds for these settings.

There are still a number of interesting questions that may worth further investigation. For example, pulling each base arm for a number of time slots at the beginning of the game can decease the constant term to non-exponential, but the point is that the player does not know how many time slots are enough. Thus how can we use an adaptive policy or some further assumptions to do so is a good question. In this paper, we suppose that all the distributions for base arms are independent. Another question is how to find analysis for using CTS under correlated arm distributions.


\bibliography{sample}

\appendix

\section{Proof of Lemmas in Section \ref{Section_Proof_T1}} \label{sec:proofTheorem1}



\subsection{Proof of Lemma \ref{Lemma45}}

{\KeyLemma*}

\begin{proof}
%
  Firstly, consider the case that we choose $Z = S^*$, i.e. we change $\bm{\theta}_{S^*}(t)$ to some $\bm{\theta'}_{S^*}$ with $||\bm{\theta'}_{S^*}- \bm{\mu}_{S^*}||_{\infty} \le \varepsilon$ and get a new vector $\bm{\theta'} = (\bm{\theta'}_{S^*}, \bm{\theta}_{{S^*}^c}(t))$.
  We claim that for any $S'$ such that $S' \cap S^* = \emptyset$, ${\sf Oracle}(\bm{\theta'}) \ne S'$. This is because
  \begin{eqnarray}
     \label{eq:z111} r(S',\bm{\theta'}) &=& r(S',\bm{\theta}(t)) \\
    \label{eq:z112} &\le& r(S(t),\bm{\theta}(t))\\
    \label{eq:z113} &\le& r(S(t),\bm{\mu}) + B({\Delta_{S(t)}\over B} - ({k^*}^2+1)\varepsilon)  \\
    \label{eq:z114} &\le& r(S^*,\bm{\mu}) - B({k^*}^2+1)\varepsilon\\
    \label{eq:00991}  &<& r(S^*,\bm{\mu}) - Bk^*\varepsilon \\
    \label{eq:z116}&\le& r(S^*,\bm{\theta'}).
  \end{eqnarray}
Eq.~\eqref{eq:z111} is because $\bm{\theta'}$ and $\bm{\theta}(t)$ only differs on arms in $S^*$ but $S' \cap S^* = \emptyset$.
Eq.~\eqref{eq:z112} is by the optimality of $S(t)$ on input $\bm{\theta}(t)$.
Eq.~\eqref{eq:z113} is by the event $\neg \mathcal{C}(t)$ and Lipschitz continuity (Assumption~\ref{Assumption_Continouos1}).
Eq.~\eqref{eq:z114} is by the definition of $\Delta_{S(t)}$.
Eq.~\eqref{eq:z116} again uses the Lipschitz continuity.
Thus, the claim holds.

So we have two possibilities for ${\sf Oracle}(\bm{\theta'})$: 1a) for all $\bm{\theta'}_{S^*}$ with $||\bm{\theta'}_{S^*}- \bm{\mu}_{S^*}||_{\infty} \le \varepsilon$,
		  $S^* \subseteq {\sf Oracle}(\bm{\theta'})$;
	  1b) for some $\bm{\theta'}_{S^*}$ with
	  $||\bm{\theta'}_{S^*}- \bm{\mu}_{S^*}||_{\infty} \le \varepsilon$,
	  ${\sf Oracle}(\bm{\theta'}) = S_1$ where $S_1 \cap S^* = Z_1$ and $Z_1 \ne S^*$, $Z_1 \ne \emptyset$.

  In 1a), let $S_0= {\sf Oracle}(\bm{\theta'})$.
  Then we have $r(S_0,\bm{\theta'}) \ge r(S^*,\bm{\theta'}) \ge r(S^*,\bm{\mu}) - Bk^*\varepsilon$.
  If $S_0 \not\in {\sf OPT}$, then we have $r(S^*,\bm{\mu}) = r(S_0,\bm{\mu}) + \Delta_{S_0}$.
  Together we have $r(S_0,\bm{\theta'}) \ge r(S_0,\bm{\mu}) + \Delta_{S_0} - Bk^*\varepsilon$.
  By the Lipschitz continuity assumption,
  this implies that $||\bm{\theta}_{S_0}' - \bm{\mu}_{S_0}||_1 \ge {\Delta_{S_0}\over B} - k^*\varepsilon > {\Delta_{S_0}\over B} - ({k^*}^2+1)\varepsilon $.
  That is, we conclude that either $S_0 \in {\sf OPT}$ or $||\bm{\theta}_{S_0}' - \bm{\mu}_{S_0}||_1 \ge {\Delta_{S_0}\over B} - k^*\varepsilon > {\Delta_{S_0}\over B} - ({k^*}^2+1)\varepsilon $, which means that $\mathcal{E}_{S^*,1}(\bm{\theta'}(t)) = \mathcal{E}_{S^*,1}(\bm{\theta}(t))$ holds.

%
   Then we consider 1b). Fix a $\bm{\theta'}_{S^*}$ with $||\bm{\theta'}_{S^*}- \bm{\mu}_{S^*}||_{\infty} \le \varepsilon$.
  Let $S_1 = {\sf Oracle}(\bm{\theta'})$ which does not equals to $S^*$.  Then $r(S_1,\bm{\theta'}) \ge r(S^*,\bm{\theta'}) \ge r(S^*,\bm{\mu}) - Bk^*\varepsilon$.  
%
%

Now we try to choose $Z = Z_1$. For all $\bm{\theta'}_{Z_1}$ with $||\bm{\theta'}_{Z_1}- \bm{\mu}_{Z_1}||_{\infty} \le \varepsilon$, consider
	$\bm{\theta'} = (\bm{\theta'}_{Z_1},\bm{\theta}_{{Z_1}^c}(t))$.
We see that $\sum_{i \in Z_1} |\theta'_{S^*} (i) -  \theta'_{Z_1}(i)|\le 2(k^*-1)\varepsilon$, where $\theta'_{S^*} (i)$ represents the $i$-th term in vector $\bm{\theta'}_{S^*}$
Thus $r(S_1,\bm{\theta'}) \ge r(S^*,\bm{\mu}) - Bk^*\varepsilon - 2B(k^*-1)\varepsilon = r(S^*,\bm{\mu}) - B(3k^*-2)\varepsilon$.
Similarly, we have the following inequalities for any $S' \cap Z_1 = \emptyset$:
  \begin{eqnarray}
    \nonumber  r(S',\bm{\theta'}) &=& r(S',\bm{\theta}(t)) \\
    \nonumber &\le& r(S(t),\bm{\theta}(t))\\
    \nonumber &\le& r(S(t),\bm{\mu}) + B({\Delta_{S(t)}\over B} - ({k^*}^2+1)\varepsilon)  \\
    \label{eq:00993} &\le& r(S^*,\bm{\mu}) - B({k^*}^2+1)\varepsilon\\
    \label{eq:00992} &<& r(S^*,\bm{\mu}) - B(3k^*-2)\varepsilon \\
    \nonumber&\le& r(S_1,\bm{\theta'}).
  \end{eqnarray}
  That is, ${\sf Oracle}(\bm{\theta'}) \cap Z_1 \ne \emptyset$.
 Thus we will also have two possibilities for ${\sf Oracle}(\bm{\theta'})$: 2a) for all $\bm{\theta'}_{Z_1}$ with
	  $||\bm{\theta'}_{Z_1}- \bm{\mu}_{Z_1}||_{\infty} \le \varepsilon$, $Z_1 \subseteq {\sf Oracle}(\bm{\theta'})$;
	2b) for some $\bm{\theta'}_{Z_1}$ with $||\bm{\theta'}_{Z_1}- \bm{\mu}_{Z_1}||_{\infty} \le \varepsilon$, ${\sf Oracle}(\bm{\theta'}) = S_2$ where $S_2 \cap Z_1= Z_2$ and $Z_2 \ne Z_1$, $Z_2 \ne \emptyset$.

We could repeat the above argument and each time the size of $Z_i$ is decreased by at least $1$. 
In the first step, the terms contain $\varepsilon$ (in Eq. \eqref{eq:00991}) is $Bk^*\varepsilon$, and in the second step, the terms contain $\varepsilon$ (in Eq. \eqref{eq:00992}) becomes $Bk^*\varepsilon + 2B(k^* - 1)\varepsilon =  B(3k^*-2)\varepsilon$.
Thus, after at most $k^*-1$ steps, this term is at most $B(k^* + 2(k^*-1) + \cdots + 2\times 1)\varepsilon = B{k^*}^2\varepsilon$, which is still less than $B({k^*}^2+1)\varepsilon$ (in Eq. \eqref{eq:z114} or \eqref{eq:00993}). This means that the above analysis works for any steps in the induction procedure. When we  reach the end, we could find a $Z_i \subseteq S^*$ and $Z_i \ne \emptyset$ such that $\mathcal{E}_{Z_i,1}(\bm{\theta}(t))$
	occurs.
%
%
\end{proof}

\subsection{Proof of Lemma \ref{Lemma3}}

Before the proof of Lemma \ref{Lemma3}, we first show the following two lemmas.

\begin{lemma}\label{Lemma1}

Let $p_{i,q,\varepsilon}$ be the probability of $|\theta_i(t) - \mu_i| \le \varepsilon$ when there are $q$ observations of base arm $i$, then for $q > {8\over \varepsilon^2}$,
  \begin{equation*}
    \E\left[{1\over p_{i,q,\varepsilon}}\right] \le 1+ 6\alpha_1'{1\over \varepsilon^2}e^{-{\varepsilon^2 \over 2}q} + {2\over e^{{1\over 8}\varepsilon^2q} - 2},
  \end{equation*} where $\alpha_1'$ is a constant which does not depend on any parameters of the MAB model, i.e. $m,k,T,M,\bm{\mu}$, and the expectation is taken over all possible $q$ observations of base arm $i$.
\end{lemma}

\begin{proof} From the definition of $p_{i,q,\varepsilon}$ and the results in  \cite{agrawal2013further}, we know that
  \begin{equation*}
    \E\left[{1\over p_{i,q,\varepsilon}}\right] = \sum_{\ell=0}^q {f_{q,\mu_i}(\ell) \over F_{q+1,\mu_i -\varepsilon}(\ell) -  F_{q+1,\mu_i +\varepsilon}(\ell)},
  \end{equation*}
  where $F_{q,p}(\ell)$ and $f_{q,p}(\ell)$ are the cumulative distribution function and probability  distribution function of binomial distribution with parameters $q,p$. $f_{q,\mu_i}(\ell)$ is the probability that there are $\ell$ positive feedbacks for $i$ in totally $q$ observations, and $F_{q+1,\mu_i -\varepsilon}(\ell) -  F_{q+1,\mu_i +\varepsilon}(\ell)$ is the probability for $|\bm{\theta}_i(t) - \mu_i| \le \varepsilon$ under this Beta Distribution $\beta(\ell+1,q-\ell+1)$.

  To calculate that, we divide it into three parts: a) $\ell=0$ to $\ell=\lfloor (\mu_i - {\varepsilon \over 2})q\rfloor$; b) $\ell=\lceil(\mu_i - {\varepsilon \over 2})q\rceil$ to $\ell=\lfloor(\mu_i + {\varepsilon \over 2})q\rfloor$; c) $\ell=\lceil(\mu_i + {\varepsilon \over 2})q\rceil$ to $\ell=q$.

  For part a), we need an inequality from \cite{agrawal2013further} (in the proof of Lemma 4 in \cite{agrawal2013further}):
  \begin{equation}\label{Eq_1221}
   \forall q > {8\over \varepsilon^2}, \sum_{\ell=0}^{\lfloor (\mu_i - {\varepsilon \over 2})q\rfloor} {f_{q,\mu_i}(\ell) \over F_{q+1,\mu_i -\varepsilon}(\ell)} \le \alpha_1'{1\over \varepsilon^2}e^{-{\varepsilon^2 \over 2}q}.
  \end{equation}

  Using Chernoff-Hoeffding inequality (Fact \ref{Fact_Chernoff}), when $q > {8\over \varepsilon^2}$, we have
  \begin{equation*}
    F_{q+1,\mu_i-\varepsilon}\left(\left\lceil(\mu_i - {\varepsilon \over 2})q\right\rceil\right) \ge 1 - e^{-2(\mu_i+{\varepsilon\over 2}q+ \varepsilon)^2/(q+1)}\ge1 - e^{-{1\over 8}\varepsilon^2q}.
  \end{equation*}

  Similarly,
  \begin{equation*}
    F_{q+1,\mu_i+\varepsilon}\left(\left\lceil(\mu_i - {\varepsilon \over 2})q\right\rceil\right) \le e^{-{1\over 8}\varepsilon^2q}.
  \end{equation*}

  Since the proportion ${f_{q+1,\mu_i-\varepsilon}(\ell) \over f_{q+1,\mu_i+\varepsilon}(\ell)} = ({1-\mu_i+\varepsilon \over 1 - \mu_i-\varepsilon})^{q+1-\ell}({\mu_i-\varepsilon \over  \mu_i+\varepsilon})^\ell$ is decreasing as $\ell$ increases. Thus the proportion of ${F_{q+1,\mu_i-\varepsilon}(\ell) \over F_{q+1,\mu_i+\varepsilon}(\ell)}$ is also decreasing.

  This means for all $\ell \le \lceil(\mu_i - {\varepsilon \over 2})q\rceil$, we have
  \begin{equation*}
    {F_{q+1,\mu_i -\varepsilon}(\ell) \over F_{q+1,\mu_i +\varepsilon}(\ell)} \ge {F_{q+1,\mu_i -\varepsilon}(\lceil(\mu_i - {\varepsilon \over 2})q\rceil) \over F_{q+1,\mu_i +\varepsilon}(\lceil(\mu_i - {\varepsilon \over 2})q\rceil)}\ge {1 - e^{-{1\over 8}\varepsilon^2q} \over e^{-{1\over 8}\varepsilon^2q}}
    \ge{3\over 2}.
  \end{equation*}

Thus
  \begin{equation*}
    \sum_{\ell=0}^{\lfloor(\mu_i - {\varepsilon \over 2})q\rfloor}{f_{q,\mu_i}(\ell) \over F_{q+1,\mu_i -\varepsilon}(\ell) - F_{q+1,\mu_i +\varepsilon}(\ell)} \le 3\sum_{\ell=0}^{\lfloor(\mu_i - {\varepsilon \over 2})q\rfloor} {f_{q,\mu_i}(\ell) \over F_{q+1,\mu_i -\varepsilon}(\ell)}
    \le 3\alpha_1'{1\over \varepsilon^2}e^{-{\varepsilon^2 \over 2}q},
  \end{equation*}
where the second inequality comes from Eq. \eqref{Eq_1221}.

  Then we consider the part b), notice that $F_{q+1,\mu_i-\varepsilon}(\ell) - F_{q+1,\mu_i+\varepsilon}(\ell)$ is first increasing and then decreasing. Thus, the minimum value of $F_{q+1,\mu_i-\varepsilon}(\ell) - F_{q+1,\mu_i+\varepsilon}(\ell)$ for $\lceil(\mu_i - {\varepsilon \over 2})q\rceil\le \ell \le \lfloor(\mu_i + {\varepsilon \over 2})q\rfloor$ is taken at the endpoints, i.e. $\ell = \lceil(\mu_i - {\varepsilon \over 2})q\rceil$ or $\ell = \lfloor(\mu_i + {\varepsilon \over 2})q\rfloor$.

  We have proved that for $q > {8\over \varepsilon^2}$, $F_{q+1,\mu_i-\varepsilon}(\lceil(\mu_i - {\varepsilon \over 2})q\rceil) - F_{q+1,\mu_i+\varepsilon}(\lceil(\mu_i - {\varepsilon \over 2})q\rceil) \ge 1 - 2e^{-{1\over 8}\varepsilon^2q}$.

  Using the same way, we can also get $F_{q+1,\mu_i-\varepsilon}(\lfloor(\mu_i + {\varepsilon \over 2})q\rfloor) - F_{q+1,\mu_i+\varepsilon}(\lfloor(\mu_i + {\varepsilon \over 2})q\rfloor) \ge 1 - 2e^{-{1\over 8}\varepsilon^2q}$.

  Thus
  \begin{eqnarray*}
    \sum_{\ell=(\lceil(\mu_i - {\varepsilon \over 2})q\rceil}^{\lfloor(\mu_i + {\varepsilon \over 2})q\rfloor}{f_{q,\mu_i}(\ell) \over F_{q+1,\mu_i -\varepsilon}(\ell) - F_{q+1,\mu_i +\varepsilon}(\ell)} &\le& \sum_{\ell=\lceil(\mu_i - {\varepsilon \over 2})q\rceil}^{\lfloor(\mu_i + {\varepsilon \over 2})q\rfloor} {f_{q,\mu_i}(\ell) \over 1 - 2e^{-{1\over 8}\varepsilon^2q}}\\
    &\le& {1\over 1 - 2e^{-{1\over 8}\varepsilon^2q}} \sum_{\ell=\lceil(\mu_i - {\varepsilon \over 2})q\rceil}^{\lfloor(\mu_i + {\varepsilon \over 2})q\rfloor} f_{q,\mu_i}(\ell)\\
    &\le& {1\over 1 - 2e^{-{1\over 8}\varepsilon^2q}}\\
    &\le& 1 + {2\over e^{{1\over 8}\varepsilon^2q} - 2}.
  \end{eqnarray*}

  Then we come to the last part, notice that $F_{q+1,\mu_i}(\ell) = 1 - F_{q+1,1-\mu_i}(q-\ell)$ and $f_{q,\mu_i}(\ell) = f_{q,1-\mu_i}(q-\ell)$. Thus
  \begin{eqnarray*}
    \sum_{\ell=\lceil(\mu_i + {\varepsilon \over 2})q\rceil}^{q}{f_{q,\mu_i}(\ell) \over F_{q+1,\mu_i -\varepsilon}(\ell) - F_{q+1,\mu_i +\varepsilon}(\ell)} &=& \sum_{\ell=\lceil(\mu_i + {\varepsilon \over 2})q\rceil}^{q}{f_{q,1-\mu_i}(q-\ell) \over F_{q+1,1-\mu_i -\varepsilon}(q-\ell) - F_{q+1,1-\mu_i +\varepsilon}(q-\ell)} \\
    &=& \sum_{\ell'=0}^{q - \lceil(\mu_i + {\varepsilon \over 2})q\rceil}{f_{q,1-\mu_i}(\ell') \over F_{q+1,1-\mu_i -\varepsilon}(\ell') - F_{q+1,1-\mu_i +\varepsilon}(\ell')} \\
    &\le& 3\alpha_1'{1\over \varepsilon^2}e^{-{\varepsilon^2 \over 2}q}.
  \end{eqnarray*}

  Summing up the three parts, for $q > {8\over \varepsilon^2}$, we have the following upper bound:
  \begin{equation*}
    \E\left[{1\over p_{i,q,\varepsilon}}\right] \le 1+ 6\alpha_1'{1\over \varepsilon^2}e^{-{\varepsilon^2 \over 2}q} + {2\over e^{{1\over 8}\varepsilon^2q} - 2}.
  \end{equation*}
\end{proof}

\begin{lemma}\label{Lemma2}
  For any value $q$, we have
  \begin{equation*}
    \E\left[{1\over p_{i,q,\varepsilon}}\right] \le {4\over \varepsilon^2}.
  \end{equation*}
\end{lemma}

\begin{proof}
  From Lemma 2 in \cite{agrawal2013further}, for any $q>0$, we have
    \begin{equation}\label{Eq_1222}
    \sum_{\ell=0}^{q} {f_{q,\mu_i}(\ell) \over F_{q+1,\mu_i -\varepsilon}(\ell)} \le {3\over \varepsilon}.
  \end{equation}

  Now we consider the value $\sum_{\ell=0}^{q} F_{q+1,\mu_i - \varepsilon}(\ell) - F_{q+1,\mu_i+\varepsilon}(\ell)$. Since $F_{q+1,\mu}(\ell+1) = \mu F_{q,\mu}(\ell) + (1-\mu)F_{q,\mu}(\ell+1)$, we have the following equations:
  \begin{eqnarray*}
  \sum_{\ell=0}^{q+1} (F_{q+1,\mu_i - \varepsilon}(\ell) - F_{q+1,\mu_i+\varepsilon}(\ell)) &=& \sum_{\ell=0}^{q} (F_{q+1,\mu_i - \varepsilon}(\ell) - F_{q+1,\mu_i+\varepsilon}(\ell))\\
  \nonumber&=&\sum_{\ell=0}^{q} ((\mu_i-\varepsilon)F_{q,\mu_i - \varepsilon}(\ell-1) - (\mu_i+\varepsilon)F_{q,\mu_i+\varepsilon}(\ell-1) \\
    & & + (1-\mu_i+\varepsilon)F_{q,\mu_i - \varepsilon}(\ell) - (1-\mu_i-\varepsilon)F_{q,\mu_i+\varepsilon}(\ell))\\
    \nonumber&=& \sum_{\ell=0}^{q}(F_{q,\mu_i - \varepsilon}(\ell) - F_{q,\mu_i+\varepsilon}(\ell))+ \varepsilon\sum_{\ell=0}^q(F_{q,\mu_i - \varepsilon}(\ell)\\
    & &- F_{q,\mu_i - \varepsilon}(\ell-1) + F_{q,\mu_i + \varepsilon}(\ell) - F_{q,\mu_i + \varepsilon}(\ell-1))\\
    &=& \sum_{\ell=0}^{q}(F_{q,\mu_i - \varepsilon}(\ell) - F_{q,\mu_i+\varepsilon}(\ell)) + 2\varepsilon.
  \end{eqnarray*}

  When $q = 1$, $F_{1,\mu-\varepsilon}(0) - F_{1,\mu+\varepsilon}(0) = 2\varepsilon$. Therefore, by induction, we know that  $\sum_{\ell=0}^{q+1} (F_{q+1,\mu_i - \varepsilon}(\ell) - F_{q+1,\mu_i+\varepsilon}(\ell)) = 2(q+1)\varepsilon$.

  Thus $\max_\ell F_{q+1,\mu_i - \varepsilon}(\ell) - F_{q+1,\mu_i+\varepsilon}(\ell) \ge 2\varepsilon$.
Let $L = \argmax_\ell F_{q+1,\mu_i - \varepsilon}(\ell) - F_{q+1,\mu_i+\varepsilon}(\ell)$, then we can write the value $\sum_{\ell=0}^{q} {f_{q,\mu_i}(\ell) \over F_{q+1,\mu_i -\varepsilon}(\ell) - F_{q+1,\mu_i +\varepsilon}(\ell)}$ as
  \begin{equation}\label{Eq_1223}
    \sum_{\ell=0}^{L-1} {f_{q,\mu_i}(\ell) \over F_{q+1,\mu_i -\varepsilon}(\ell) - F_{q+1,\mu_i +\varepsilon}(\ell)} + \sum_{\ell=L+1}^{q} {f_{q,\mu_i}(\ell) \over F_{q+1,\mu_i -\varepsilon}(\ell) - F_{q+1,\mu_i +\varepsilon}(\ell)} + {f_{q,\mu_i}(L) \over F_{q+1,\mu_i -\varepsilon}(L) - F_{q+1,\mu_i +\varepsilon}(L)}.
  \end{equation}

 Consider the value of first term, since ${F_{q+1,\mu_i -\varepsilon}(\ell) \over F_{q+1,\mu_i +\varepsilon}(\ell)}$ is decreasing, then for $0 \le \ell \le L-1$, ${F_{q+1,\mu_i -\varepsilon}(\ell) \over F_{q+1,\mu_i +\varepsilon}(\ell)} \ge {F_{q+1,\mu_i -\varepsilon}(L) \over F_{q+1,\mu_i +\varepsilon}(L)} \ge {1\over 1-2\varepsilon}$.

  Thus
  \begin{eqnarray*}
    \sum_{\ell=0}^{L-1} {f_{q,\mu_i}(\ell) \over F_{q+1,\mu_i -\varepsilon}(\ell) - F_{q+1,\mu_i +\varepsilon}(\ell)} &\le& {1\over 2\varepsilon}\sum_{\ell=0}^{L-1} {f_{q,\mu_i}(\ell) \over F_{q+1,\mu_i -\varepsilon}(\ell)}  \\
    &\le& {1\over 2\varepsilon} \sum_{\ell=0}^{q} {f_{q,\mu_i}(\ell) \over F_{q+1,\mu_i -\varepsilon}(\ell)}\\
    &\le& {3\over 2\varepsilon^2},
  \end{eqnarray*}
where the last inequality comes from Eq. \eqref{Eq_1222}.

Similarly, for the second term in Eq. \eqref{Eq_1223}, we know that
  \begin{eqnarray*}
    \sum_{\ell=L+1}^{q} {f_{q,\mu_i}(\ell) \over F_{q+1,\mu_i -\varepsilon}(\ell) - F_{q+1,\mu_i +\varepsilon}(\ell)} &=& \sum_{\ell=L+1}^{q}{f_{q,1-\mu_i}(q-\ell) \over F_{q+1,1-\mu_i -\varepsilon}(q-\ell) - F_{q+1,1-\mu_i +\varepsilon}(q-\ell)}\\
    &=& \sum_{\ell'=0}^{q-L-1}{f_{q,1-\mu_i}(\ell') \over F_{q+1,1-\mu_i -\varepsilon}(\ell') - F_{q+1,1-\mu_i +\varepsilon}(\ell')}.
  \end{eqnarray*}

  Since $F_{q+1,1-\mu_i -\varepsilon}(\ell') - F_{q+1,1-\mu_i +\varepsilon}(\ell')$ equals to $F_{q+1,1-\mu_i -\varepsilon}(q-\ell) - F_{q+1,1-\mu_i +\varepsilon}(q-\ell)$, we still have that:
  \begin{eqnarray*}
    \sum_{\ell'=0}^{q-L-1}{f_{q,1-\mu_i}(\ell') \over F_{q+1,1-\mu_i -\varepsilon}(\ell'+1) - F_{q+1,1-\mu_i +\varepsilon}(\ell'+1)} &\le& {1\over 2\varepsilon}\sum_{\ell'=0}^{q-L-1}{f_{q,1-\mu_i}(\ell') \over F_{q+1,1-\mu_i -\varepsilon}(\ell')}\\
    &\le& {1\over 2\varepsilon}\sum_{\ell'=0}^{q}{f_{q,1-\mu_i}(\ell') \over F_{q+1,1-\mu_i -\varepsilon}(\ell')}\\
    &\le&{3\over 2\varepsilon^2}.
  \end{eqnarray*}

  From the fact that $F_{q+1,\mu_i - \varepsilon}(L) - F_{q+1,\mu_i+\varepsilon}(L) \ge 2\varepsilon$, we know  the last term in Eq. \eqref{Eq_1223} satisfies ${f_{q,\mu_i}(L) \over F_{q+1,\mu_i -\varepsilon}(L) - F_{q+1,\mu_i +\varepsilon}(L)} \le {1\over 2\varepsilon}$.

  Thus
  \begin{equation*}
    \E\left[{1\over p_{i,q,\varepsilon}}\right] \le {3\over 2\varepsilon^2} + {3\over 2\varepsilon^2} + {1\over 2\varepsilon} \le {4\over \varepsilon^2}.
  \end{equation*}
\end{proof}

Now we provide the proof of Lemma \ref{Lemma3}.

{\LemmaThree*}

\begin{proof} First we consider a setting that we only update the prior distribution when $\mathcal{E}_{Z,1}(\bm{\theta}(t))$ occurs but $\mathcal{E}_{Z,2}(\bm{\theta}(t))$ does not occur, i.e., for a fixed $Z \subseteq S^*$, before running the {\sf Update} procedure in line 5 of Algorithm \ref{Algorithm_TS}, we first check whether the sample vector $\bm{\theta}(t)$ satisfies that $\mathcal{E}_{Z,1}(\bm{\theta}(t))$ happens but $\mathcal{E}_{Z,2}(\bm{\theta}(t))$ does not happen. Only if the answer is yes, we run the {\sf Update} procedure, and otherwise we skip the {\sf Update} procedure.
Note that we will have a feedback for all $i\in Z$ when $\mathcal{E}_{Z,1}(\bm{\theta}(t))$ occurs but $\mathcal{E}_{Z,2}(\bm{\theta}(t))$ does not occur. 
Let these time steps be $\tau_1,\tau_2,\cdots$ (let $\tau_0 = 0$).
  Then at time $\tau_{q}+1$, the number of feedbacks for any $i\in Z$ is $q$.

  Let $\bm{N}_Z(t) = [N_{z_1}(t),\cdots,N_{z_k}(t)]$ denote the number of feedbacks of base arms in $Z$ until time slot $t$, and $p_{Z,\bm{N}_Z(t),\varepsilon}$ be the probability that $\mathcal{E}_{Z,2}(\bm{\theta}(t))$ does not occur at time $t$.
Note that $\mathcal{E}_{Z,1}(\bm{\theta}(t)), \mathcal{E}_{Z,2}(\bm{\theta}(t))$ are independent events conditioned on the
	  observation history $\mathcal{F}_{t-1}$ (since $\mathcal{E}_{Z,1}(\bm{\theta}(t))$ only depends on $\bm{\theta}_{Z^c}(t)$ and $\mathcal{E}_{Z,2}(\bm{\theta}(t))$ only depends on $\bm{\theta}_{Z}(t)$, and $\bm{\theta}_{Z^c}(t)$, $\bm{\theta}_{Z}(t)$ are independent random variables conditioned on $\mathcal{F}_{t-1}$), and the observation history does not change from time steps
	  $\tau_q+1$ to $\tau_{q+1}-1$, $p_{Z,\bm{N}_Z(\tau_q+1),\varepsilon}$ is also the probability that
	  $\mathcal{E}_{Z,2}(\bm{\theta}(t))$ does not occur conditioned on $\mathcal{E}_{Z,1}(\bm{\theta}(t))$ occurs for any $t$
	  between $\tau_q +1 $ and $\tau_{q+1}$.
	  Therefore, for given $\mathcal{F}_{\tau_q}$, the vector $\bm{N}_Z(\tau_q+1)$ is fixed, and the expected number of time steps for both $\mathcal{E}_{Z,1}(\bm{\theta}(t))$ and $\mathcal{E}_{Z,2}(\bm{\theta}(t))$ occur from $\tau_q+1$ to $\tau_{q+1}$ is ${1\over p_{Z,\bm{N}_Z(\tau_q+1),\varepsilon}} - 1$.
	  Take another expectation over the observation history $\mathcal{F}_{\tau_q}$, the expected number of time steps for both $\mathcal{E}_{Z,1}(\bm{\theta}(t))$ and $\mathcal{E}_{Z,2}(\bm{\theta}(t))$ occur from $\tau_q+1$ to $\tau_{q+1}$ is $\E[{1\over p_{Z,\bm{N}_Z(\tau_q+1),\varepsilon}}] - 1$.
	  

  Recall that event $\neg\mathcal{E}_{Z,2}(\bm{\theta}(t)) = \{||\bm{\theta}_Z(t) - \bm{\mu}_Z||_\infty \le \varepsilon \}
	  = \{\forall i\in Z, |\theta_i(t) - \mu_i| \le \varepsilon \}$. Then we can write $\E[{1\over p_{Z,\bm{N}_Z(\tau_q+1),\varepsilon}}]$ as:
  \begin{equation*}
\E [{1\over p_{Z,\bm{N}_Z(\tau_q+1),\varepsilon}}] = \sum_{\bm{a}_Z(\tau_q+1), \bm{b}_Z(\tau_q+1)} {\Pr\left[\I[\bm{a}_Z(\tau_q+1),\bm{b}_Z(\tau_q+1),\bm{N}_Z(\tau_q+1)] \right]\over \Pr[\neg \mathcal{E}_{Z,2}(\bm{\theta}(\tau_q+1)) | \bm{a}_Z(\tau_q+1),\bm{b}_Z(\tau_q+1),\bm{N}_Z(\tau_q+1)]},
\end{equation*}where $\bm{a}_Z(t), \bm{b}_Z(t)$ represent the value vector of $a_i(t)$ and $b_i(t)$ in base arm set $Z$, respectively; and  $\Pr\left[\I[\bm{a}_Z(\tau_q+1),\bm{b}_Z(\tau_q+1),\bm{N}_Z(\tau_q+1)] \right]$ is the probability that $N_i(\tau_q+1)$ observations form pair $(a_i(\tau_q+1),b_i(\tau_q+1))$ for any arm $i\in Z$.

Since we draw independent samples in each time slot, and the sample $\theta_i(t)$ only depends on $a_i(t),b_i(t)$, we have 
\begin{eqnarray*}
&&\Pr\left[\neg\mathcal{E}_{Z,2}(\bm{\theta}(\tau_q+1)) | \bm{a}_Z(\tau_q+1),\bm{b}_Z(\tau_q+1),\bm{N}_Z(\tau_q+1)\right]\\
&=&\prod_{i\in Z} \Pr\left[ |\theta_i(\tau_q+1) - \mu_i| \le \varepsilon | a_i(\tau_q+1),b_i(\tau_q+1)\right].
\end{eqnarray*}

Only if for all the arms $i\in Z$, the first $N_i(\tau_q+1)$ observations $\{Y_i(\tau_k)\}_{k=1}^q$ contains $a_i(\tau_q+1)-1$ 1's and $b_i(\tau_q+1)-1$ 0's, $\I[\bm{a}_Z(\tau_q+1),\bm{b}_Z(\tau_q+1),\bm{N}_Z(\tau_q+1)]$ equals to 1. Then by Assumption \ref{Assumption_IND}, we know that:
\begin{equation*}
\Pr\left[\I[\bm{a}_Z(\tau_q+1),\bm{b}_Z(\tau_q+1),\bm{N}_Z(\tau_q+1)] \right] \le \prod_{i\in Z} {N_i(\tau_q+1)\choose a_i(\tau_q +1) - 1} \mu_i^{a_i(\tau_q +1) - 1}(1-\mu_i)^{b_i(\tau_q +1) - 1}.
\end{equation*}

Thus
\begin{eqnarray*}
&&\E [{1\over p_{Z,\bm{N}_Z(\tau_q+1),\varepsilon}}] \\
&\le& \sum_{\bm{a}_Z(\tau_q+1), \bm{b}_Z(\tau_q+1)} \left(\prod_{i\in Z} {{N_i(\tau_q+1)\choose a_i(\tau_q +1) - 1} \mu_i^{a_i(\tau_q +1) - 1}(1-\mu_i)^{b_i(\tau_q +1) - 1} \over \Pr\left[ |\theta_i(\tau_q+1) - \mu_i| \le \varepsilon | a_i(\tau_q+1),b_i(\tau_q+1)\right]}\right)\\
&=& \prod_{i\in Z} \left(\sum_{a_i(\tau_q+1),b_i(\tau_q+1)} {{N_i(\tau_q+1)\choose a_i(\tau_q +1) - 1} \mu_i^{a_i(\tau_q +1) - 1}(1-\mu_i)^{b_i(\tau_q +1) - 1} \over \Pr\left[ |\theta_i(\tau_q+1) - \mu_i| \le \varepsilon | a_i(\tau_q+1),b_i(\tau_q+1)\right]}\right)\\
&=& \prod_{i\in Z}\E[{1\over p_{i,q,\varepsilon}}].
\end{eqnarray*} where $p_{i,q,\varepsilon}$ is the probability of $|\theta_{i}(t) - \mu_{i}| \le \varepsilon$ when there are $q$
   observations of base arm $i$.

  
In the following, we choose $R = {8\over \varepsilon^2}(\log(12\alpha_1'|Z|+4|Z|+2) + \log {1\over \varepsilon^2}) > {8\over \epsilon^2}$ and let $B_q$ be the upper bound for $\E[{1\over p_{i,q,\varepsilon}}]$ (when $q > R$, use the bound from Lemma \ref{Lemma1}, when $q \le R$, use the bound from Lemma \ref{Lemma2}).

   Then in this setting, we have 
  \begin{eqnarray*}
    &&\sum_{t=1}^T \E[\I\{\mathcal{E}_{Z,1}(\bm{\theta}(t)), \mathcal{E}_{Z,2}(\bm{\theta}(t))\}] \\
    &\le& \sum_{q=0}^T \left(\prod_{i\in Z}\E\left[{1\over p_{i,q,\varepsilon}}\right]-1\right) \\
    &\le& \sum_{q=0}^T \left(\prod_{i\in Z}B_q-1\right)\\
    &\le& \sum_{q=0}^{R}\left({4\over \varepsilon^2}\right)^{|Z|} + \sum_{q=R+1}^\infty \left(\prod_{i \in Z} \left(1 + 6\alpha_1'{1\over \varepsilon^2}e^{-{\varepsilon^2 \over 2}q} + {2\over e^{{1\over 8}\varepsilon^2q} - 2}\right) - 1\right)\\
    &\le& R\left({4\over \varepsilon^2}\right)^{|Z|} + \sum_{q=R+1}^\infty \left(1 + |Z|^2\left(6\alpha_1'{1\over \varepsilon^2}e^{-{\varepsilon^2 \over 2}q} + {2\over e^{{1\over 8}\varepsilon^2q} - 2}\right) - 1\right)\\
    &=& R\left({4\over \varepsilon^2}\right)^{|Z|} +  \sum_{q=R + 1}^{\infty}|Z|^2\left(6\alpha_1'{1\over \varepsilon^2}e^{-{\varepsilon^2 \over 2}q} + {2\over e^{{1\over 8}\varepsilon^2q} - 2}\right)\\
  &\le& R\left({4\over \varepsilon^2}\right)^{|Z|} + |Z|^2\left({12\alpha_1'\over \varepsilon^4} + {24\over \varepsilon^2}\right)\\
  &=&{2^{2|Z|+3}(\log(12\alpha_1'|Z|+4|Z|+2) + \log {1\over \varepsilon^2}) \over \varepsilon^{2|Z|+2}} + {12|Z|^2\alpha_1'\over \varepsilon^4} + {24|Z|^2\over \varepsilon^2}\\
  &\le& 13\alpha'_2 \cdot \left({2^{2|Z|+3}\log{|Z|\over \varepsilon^2} \over \varepsilon^{2|Z|+2}}\right).
  \end{eqnarray*}

  In the real setting, the priors in $Z$ can be updated at any time step. Then the expected time slots for both $\mathcal{E}_{Z,1}(\bm{\theta}(t))$ and $\mathcal{E}_{Z,2}(\bm{\theta}(t))$ occur from $\tau_q+1$ to $\tau_{q+1}$ is not $\E [{1\over p_{Z,\bm{N}_Z(\tau_q+1),\varepsilon}}]-1$, but a weighted mean of $\E [{1\over p_{Z,\bm{N}'_Z,\varepsilon}}]-1$ for all $\bm{N}'_Z \ge \bm{N}_Z(\tau_q+1)$ (here $\bm{N}' \ge \bm{N}$ means that for any $i$, $N_i' \ge N_i$).

  Notice that $\E [{1\over p_{Z,\bm{N}'_Z,\varepsilon}}] \le \prod_{i\in Z} \E[{1\over p_{i,N'_i,\varepsilon}}]$ and $B_q$ is non-increasing, then we have
  \begin{equation*}
    \E [{1\over p_{Z,\bm{N}'_Z,\varepsilon}}]= \prod_{i\in Z} \E[{1\over p_{i,N'_i,\varepsilon}}]
    \le \prod_{i\in Z} B_{N'_i}
    \le \prod_{i\in Z} B_{N_i(\tau_q+1)}
    \le \prod_{i\in Z} B_q.
  \end{equation*}

  Although we do not know what the exact weights are, we can see that $\sum_{q=0}^T \left(\prod_{i\in Z}B_q-1\right)$ is still an upper bound. Thus, in the real setting, $\sum_{t=1}^T \E[\I\{\mathcal{E}_{Z,1}(\bm{\theta}(t)),\mathcal{E}_{Z,2}(\bm{\theta}(t))\}]$ should still be smaller than $13\alpha'_2 \left({2^{2|Z|+3}\log{|Z|\over \varepsilon^2} \over \varepsilon^{2|Z|+2}}\right)$.
\end{proof}

\section{Proof of Theorem \ref{Theorem_TS_MB}}\label{Section_A3}

In this section, we first recall Fact \ref{Fact_Bijection_zheng}.

{\FactKWAEE*}

Under Fact \ref{Fact_Bijection_zheng}, with a bijection $L_t$, we could decouple the regret of playing one action $S(t)$ to each pair of mapped arms between
	$S(t)$ and $S^*$, i.e., the regret of time $t$ is $\sum_{k=1}^K \mu_{L_t(k)} - \mu_{i^{(k)}(t)}$.

We use $N_{i,j}(t)$ to denote the number of rounds that $i^{(k)}(t) = i$ and $L_t(k) = j$ for $i\notin S^*$, $j\in S^*$ within in time slots $1,2,\cdots,T$, then
\begin{equation*}
  Reg(T) \le \sum_{i\notin S^*} \sum_{j: j\in S^*, \mu_j > \mu_i} \E[N_{i,j}(t)](\mu_j - \mu_i).
\end{equation*}

We can see that if $\{j: j\in S^*, \mu_j > \mu_i\} = \emptyset$, then $\sum_{i\notin S^*} \sum_{j: j\in S^*, \mu_j > \mu_i} \E[N_{i,j}(t)](\mu_j - \mu_i) = 0$, thus we do not need to consider the regret from base arm $i$, so we set $\Delta_i = \infty$ to make ${1\over \Delta_i} = 0$.

Now we just need to bound the value $N_{i,j}(t)$, similarly, we can defined the following three events:

\begin{itemize}
\item $\mathcal{A}_{i,j}(t) = \{\exists k, i^{(k)}(t) = i \land L_t(k) = j\}$;
\item $\mathcal{B}_{i}(t) = \{\hat{\mu}_i(t) > \mu_i + \varepsilon\}$;
\item $\mathcal{C}_{i,j}(t) = \{\theta_i(t) > \mu_j - \varepsilon\}$.
\end{itemize}

Thus \begin{eqnarray*}
\E[N_{i,j}(t)]&= &\E\left[\sum_{t=1}^T \I[\mathcal{A}_{i,j}(t) ]\right] \\
&\le& \E\left[\sum_{t=1}^T \I[\mathcal{A}_{i,j}(t) \land \mathcal{B}_{i}(t)]\right] + \E\left[\sum_{t=1}^T \I[\mathcal{A}_{i,j}(t) \land \neg \mathcal{B}_{i}(t)\land \mathcal{C}_{i,j}(t) ]\right] \\
&&+ \E\left[\sum_{t=1}^T \I[\mathcal{A}_{i,j}(t)\land  \neg\mathcal{C}_{i,j}(t)]\right].
\end{eqnarray*}





We now show some lemmas, and then provide the complete proof of Theorem \ref{Theorem_TS_MB}.

\subsection{Proof of Some Lemmas}

\begin{lemma}\label{Fact_M5prime}
In Algorithm \ref{Algorithm_TS}, for any base arm $i$, we have the following two inequalities:
\begin{equation*}
  \Pr\left[\theta_i(t)-\hat{\mu}_i(t)  > \sqrt{2\log T\over N_i(t)}\right] \le {1\over T},
\end{equation*}
\begin{equation*}
  \Pr\left[\hat{\mu}_i(t) - \theta_i(t) > \sqrt{2\log T\over N_i(t)}\right] \le {1\over T}.
\end{equation*}
\end{lemma}

\begin{proof}
The proof is the same as Lemma \ref{Fact_M5}.
\end{proof}


\begin{restatable}{lemma}{KeyLemmatwo}\label{Lemma46}
  Suppose the vector $\bm{\theta}(t)$ satisfy that $\mathcal{A}_{i,j}(t)\land \neg\mathcal{C}_{i,j}(t)$ happens. Then if we change $\theta_j(t)$ to $\theta_j'(t) > \mu_j - \varepsilon$ and set other values in $\bm{\theta}(t)$ unchanged to get $\bm{\theta}'(t)$, arm $j$ must be chosen in $\bm{\theta}'(t)$.
\end{restatable}

\begin{proof} Since the oracle use greedy algorithm to get the result, then by Fact \ref{Fact_Bijection_zheng}, there are two possibilities:
  (a) in the greedy algorithm, the steps remain until that one to choose arm $i$, (b) those steps have changed because of $\theta_j(t)$ is modified to $\theta_j'(t)$.

  If (a) happens, as now $\theta_j'(t) > \mu_j - \varepsilon \ge \theta_i(t)$, we will choose arm $j$ instead of arm $i$.

  If (b) happens, notice that arm $j$ is always available during all the previous steps, and only its sample value becomes larger. So the only way is to choose arm $j$ earlier.
\end{proof}

We use the notation $\bm{\theta}_{-i}$ to be the vector $\bm{\theta}$ without $\theta_i$.

For any $j \in S^*$, let $W_{j}$ be the set of all possible values of $\bm{\theta}$ satisfies that $\mathcal{A}_{i,j}(t)\land \neg\mathcal{C}_{i,j}(t)$ happens for some $i$, and $W_{-j} = \{\bm{\theta}_{-j}: \bm{\theta} \in W_{j}\}$.



%
%

\begin{lemma}\label{Lemma_Independent}
  In Algorithm \ref{Algorithm_TS}, for any $j\in S^*$, we have
  \begin{equation*}
    \E\left[\sum_{t=1}^T \Pr[\bm{\theta}(t) \in W_{j} | \mathcal{F}_{t-1}]\right] \le \alpha_2' \cdot {1\over \varepsilon^4},
  \end{equation*}
  where $\alpha_2'$ is a constant not dependent on the problem instance.
\end{lemma}

\begin{proof} By Lemma \ref{Lemma46}, we have that $\{\bm{\theta}(t) \in W_{j}\} \to \{\theta_j(t) \le \mu_j -\varepsilon, \bm{\theta}_{-j}(t) \in W_{-j}\}$, thus we can use
  \begin{equation*}
    \E\left[\sum_{t=1}^T \Pr[\theta_j(t) \le \mu_j -\varepsilon, \bm{\theta}_{-j}(t) \in W_{-j} | \mathcal{F}_{t-1}]\right]
  \end{equation*}
  as an upper bound.

  Denote the value $p_{j,t}$ as $\Pr[\theta_j(t) > \mu_j - \varepsilon| \mathcal{F}_{t-1}]$. Notice that given $\mathcal{F}_{t-1}$, the value $\theta_j(t)$ and the value set $\bm{\theta}_{-j}(t)$ are independent in our Algorithm \ref{Algorithm_TS}, i.e.,
  \begin{equation*}
    \Pr[\theta_j(t) \le \mu_j -\varepsilon, \bm{\theta}_{-j}(t) \in W_{-j}|\mathcal{F}_{t-1}] = p_{j,t}\Pr[\bm{\theta}_{-j}(t) \in W_{-j}|\mathcal{F}_{t-1}].
  \end{equation*}

  Then we can use Lemma 1 in \cite{agrawal2013further}, which implies that
  \begin{equation*}
    \Pr[\theta_j(t) \le \mu_j -\varepsilon, \bm{\theta}_{-j}(t) \in W_{-j}| \mathcal{F}_{t-1}]\le {1-p_{j,t}\over p_{j,t}}\Pr[j\in S(t), \bm{\theta}_{-j}(t) \in W_{-j}| \mathcal{F}_{t-1}].
  \end{equation*}

  Then \begin{eqnarray}
         \nonumber\E\left[\sum_{t=1}^T \Pr[\bm{\theta}(t) \in W_{j} | \mathcal{F}_{t-1} ]\right]
         &\le&\E\left[\sum_{t=1}^T \Pr[\theta_j(t) \le \mu_j -\varepsilon, \bm{\theta}_{-j}(t) \in W_{-j} | \mathcal{F}_{t-1}]\right]  \\
         \nonumber&\le&\E\left[\sum_{t=1}^T {1-p_{j,t}\over p_{j,t}}\Pr[j\in S(t), \bm{\theta}_{-j}(t) \in W_{-j}| \mathcal{F}_{t-1}]\right] \\
          \nonumber&=& \sum_{t=1}^T  \E\left[{1-p_{j,t}\over p_{j,t}}\I[j\in S(t), \bm{\theta}_{-j}(t) \in W_{-j}]\right] \\
          \nonumber&\le& \sum_{t=1}^T\E\left[{1-p_{j,t}\over p_{j,t}} \I[j\in S(t)]\right] \\
          \label{eq:z21}&\le& \sum_{q=0}^{T-1} \E\left[{1-p_{j,\tau_{j,q}+1}\over p_{j,\tau_{j,q}+1}}\sum_{t=\tau_{j,q}+1}^{\tau_{j,q+1}}  \mathbb{I}(i^{(k)}(t) = j)]\right] \\
          \nonumber&=& \sum_{q=0}^{T-1} \E\left[{1-p_{j,\tau_{j,q}}\over p_{j,\tau_{j,q}}}\right],
       \end{eqnarray}
  where $\tau_{j,q}$ is the time step that arm $j$ is observed for the $q$-th time.

  Eq. \eqref{eq:z21} is because the fact that the probability $\Pr[\theta_j(t) > \mu_j - \varepsilon| \mathcal{F}_{t-1}]$ only changes when we get a feedback of base arm $j$, but during time slots $[\tau_{j,q}+1,\tau_{j,q+1}]$, we do not have any such feedback.

  From analysis in \cite{agrawal2013further}, we know $\sum_{q=0}^{T-1} \E[{1-p_{j,\tau_{j,q}}\over p_{j,\tau_{j,q}}}] \le \alpha_2' \cdot {1\over \varepsilon^4}$ for some constant $\alpha_2'$ that is not dependent on the problem instance, thus
  \begin{equation*}
    \E\left[\sum_{t=1}^T \Pr[\bm{\theta}(t) \in W_{j} | \mathcal{F}_{t-1}]\right] \le \alpha_2' \cdot {1\over \varepsilon^4}.
  \end{equation*}
\end{proof}


\begin{fact}(Lemma $3$ in \cite{KWAEE14})\label{Lemma_COM}
  For any $\Delta_1 \ge \Delta_2 \ge \cdots \ge \Delta_K > 0$,
  \begin{equation*}
    \Delta_1{1\over \Delta_1^2} + \sum_{k=2}^K \Delta_k\left({1\over \Delta_k^2} - {1\over \Delta_{k-1}^2}\right) \le {2\over \Delta_K}.
  \end{equation*}
\end{fact}

\subsection{Main Proof of Theorem \ref{Theorem_TS_MB}}

Now we provide the main proof of Theorem \ref{Theorem_TS_MB}.

{\TheoremFive*}

\begin{proof} With a bijection $L_t$, we could decouple the regret of playing one action $S(t)$ to each pair of mapped arms between
	$S(t)$ and $S^*$. For example, the regret of time $t$ is $\sum_{k=1}^K \mu_{L_t(k)} - \mu_{i^{(k)}(t)}$.

We use $N_{i,j}(t)$ to denote the number of rounds that $i^{(k)}(t) = i$ and $L_t(k) = j$ for $i\notin S^*$, $j\in S^*$ within in time slots $1,2,\cdots,T$, then
\begin{eqnarray*}
  \E\left[\sum_{t=1}^T r(S^*,\mu) - r(S(t),\mu)\right] &=& \E\left[\sum_{t=1}^T \sum_{k=1}^K (\mu_{L_t(k)} - \mu_{i^{(k)}(t)})\right] \\
  &=&  \sum_{i\notin S^*} \sum_{j\in S^*} \E[N_{i,j}(t)](\mu_j - \mu_i)\\
  &\le& \sum_{i\notin S^*} \sum_{j: j\in S^*, \mu_j > \mu_i} \E[N_{i,j}(t)](\mu_j - \mu_i).
\end{eqnarray*}

We can see that if $\{j: j\in S^*, \mu_j > \mu_i\} = \emptyset$, then $\sum_{i\notin S^*} \sum_{j: j\in S^*, \mu_j > \mu_i} \E[N_{i,j}(t)](\mu_j - \mu_i) = 0$, thus we do not need to consider the regret from base arm $i$, so we set $\Delta_i = \infty$ to make ${1\over \Delta_i} = 0$.

Now we bound the value $N_{i,j}(t)$, recall that 

\begin{itemize}
\item $\mathcal{A}_{i,j}(t) = \{i^{(k)}(t) = i \land L_t(k) = j\}$;
\item $\mathcal{B}_{i}(t) = \{\hat{\mu}_i(t) > \mu_i + \varepsilon\}$;
\item $\mathcal{C}_{i,j}(t) = \{\theta_i(t) > \mu_j - \varepsilon\}$.
\end{itemize}

Then we have
\begin{eqnarray*}
\E[N_{i,j}(t)] &=& \E\left[\sum_{t=1}^T \I[\mathcal{A}_{i,j}(t)]\right]\\
&\le& \E\left[\sum_{t=1}^T \I[\mathcal{A}_{i,j}(t) \land \mathcal{B}_{i}(t)]\right] + \E\left[\sum_{t=1}^T \I[\mathcal{A}_{i,j}(t) \land \neg \mathcal{B}_{i}(t)\land \mathcal{C}_{i,j}(t) ]\right] \\
&&+ \E\left[\sum_{t=1}^T \I[\mathcal{A}_{i,j}(t)\land \neg\mathcal{C}_{i,j}(t)]\right].
\end{eqnarray*}



\noindent\textbf{The first term:}

Summing over all possible pairs $(i,j)$, the first term has upper bound 
\begin{eqnarray*}
\sum_{i \notin S, j\in S} \E\left[\sum_{t=1}^T \I[\mathcal{A}_{i,j}(t) \land \mathcal{B}_{i}(t)]\right] &=&\sum_{i \notin S} \E\left[\sum_{t=1}^T \sum_{j\in S}\I[\mathcal{A}_{i,j}(t) \land \mathcal{B}_{i}(t)]\right]\\
&\le& \sum_{i \notin S } \E\left[\sum_{t=1}^T\I[\mathcal{B}_{i}(t) \land \{i \in S(t)\}]\right].
\end{eqnarray*}

From Lemma \ref{Lemma_M1}, we can see that $\E[|t:1\le t \le T, i\in S(t), |\hat{\mu}_i(t) - \mu_i| > \varepsilon|] \le 1+{1\over \varepsilon^2}$ for any base arm $i \notin S^*$.  Summing all the base arms up, the upper bound is $(m-K)(1+{1\over \varepsilon^2})$.

\noindent\textbf{The second term:}

In the second term, let $\mu_{j_1} \ge \mu_{j_2} \ge \cdots \ge \mu_{j_{i(s)}} > \mu_i \ge \mu_{j_{s(i)+1}} \ge \cdots \ge \mu_{j_K}$, where $S^* = \{j_1,\cdots,j_K\}$, and $s(i)$ is the number of base arms in $S^*$ with mean larger than $\mu_i$, then we can write this term as
\begin{eqnarray*}
\sum_{i \notin S, j\in S} (\mu_j - \mu_i)\E\left[\sum_{t=1}^T \I[\mathcal{A}_{i,j}(t) \land \neg \mathcal{B}_{i}(t)\land \mathcal{C}_{i,j}(t) ]\right].
\end{eqnarray*}

Notice that with probability $1-{1\over T}$, $\theta_i(t) < \hat{\mu}_i(t) + \sqrt{2\log T\over N_i(t)}$ (Lemma \ref{Fact_M5prime}), thus in expectation, the time steps that $\mathcal{A}_{i,j}(t) \land \neg \mathcal{B}_{i}(t)\land \mathcal{C}_{i,j}(t)$ occurs is upper bounded by ${2\log T \over (\mu_{j_k} - \mu_i - 2\varepsilon)^2} + 1$

Then the total expected regret due to choosing $i$ in $S(t)$ while $\mathcal{A}_{i,j}(t) \land \neg \mathcal{B}_{i}(t)\land \mathcal{C}_{i,j}(t)$ occurs is upper bounded by
\begin{equation}\label{eq111}
  2\log T\cdot \left({\mu_{j_1} - \mu_i\over (\mu_{j_1} - \mu_i - 2\varepsilon)^2} + \sum_{n=2}^{s(i)} (\mu_{j_n} - \mu_i)\left({1\over (\mu_{j_n} - \mu_i - 2\varepsilon)^2} - {1\over (\mu_{j_{n-1}} - \mu_i - 2\varepsilon)^2}\right)\right) + K.
\end{equation}

The reason is that for any value of $k$, $\theta_i(t)$ can be larger than $\mu_{j_k} - \varepsilon$ for at most the first ${\log T \over 2(\mu_{j_k} - \mu_i - 2\varepsilon)^2}$ times, and the first term of Eq. \eqref{eq111} is the largest regret satisfying this constraint. The second term of Eq. \eqref{eq111} is the expectation on the small error probability.

By Fact \ref{Lemma_COM}, we have the following upper bound for that:
\begin{equation*}
  2\log T\cdot {2\Delta_i - 4\varepsilon + 2\varepsilon \over (\Delta_i-2\varepsilon)^2} + K = {4\log T\over \Delta_i - 2\varepsilon}{\Delta_i - \varepsilon \over \Delta_i - 2\varepsilon} + K.
\end{equation*}

\noindent\textbf{The third term:}

Lemma \ref{Lemma_Independent} shows that $\E\left[\sum_{t=1}^T \Pr[\bm{\theta}(t) \in W_{j}]\right] \le \alpha_2' \cdot {1\over \varepsilon^4}$. Then from definition of $W_{j}$, we have
\begin{eqnarray*}
\sum_{i \notin S, j\in S} \E\left[\sum_{t=1}^T \I[\mathcal{A}_{i,j}(t)\land \neg\mathcal{C}_{i,j}(t)]\right] &=& \sum_{j\in S}  \E\left[\sum_{t=1}^T \sum_{i\notin S}\I[\mathcal{A}_{i,j}(t)\land \neg\mathcal{C}_{i,j}(t)]\right]\\
&\le& \sum_{j\in S}\E\left[\sum_{t=1}^T \I [\bm{\theta}(t) \in W_j]\right]\\
&\le&\alpha_2' \cdot {K\over \varepsilon^4}.
\end{eqnarray*}

\noindent\textbf{Sum of all the terms:}

Thus, the total regret upper bound is:
\begin{eqnarray*}
  \nonumber\sum_{i\notin S^*} \sum_{j: j\in S^*, \mu_j > \mu_i} \E[N_{i,j}(t)](\mu_j - \mu_i) &\le& \sum_{i\notin S^*}{4\log T\over \Delta_i - 2\varepsilon}{\Delta_i - \varepsilon \over \Delta_i - 2\varepsilon} + \alpha_2' \cdot {K\over \varepsilon^4} + (m-K)\left({1\over \varepsilon^2} +K +1\right) \\
  &\le& \sum_{i\notin S^*}{4\log T\over \Delta_i - 2\varepsilon}{\Delta_i - \varepsilon \over \Delta_i - 2\varepsilon} + (\alpha_2' + 1)\cdot{m\over \varepsilon^4} + m^2.
\end{eqnarray*}

Let $\alpha_2 = \alpha_2' + 1$, and we know it is also a constant not dependent on the problem instance.

\end{proof}

\end{document}